\documentclass{article}
\usepackage{iclr2027_conference,times}

\usepackage{amsmath,amsfonts,bm}

\def\eqref#1{equation~\ref{#1}}

\def\1{\bm{1}}

\DeclareMathAlphabet{\mathsfit}{\encodingdefault}{\sfdefault}{m}{sl}
\SetMathAlphabet{\mathsfit}{bold}{\encodingdefault}{\sfdefault}{bx}{n}

\newcommand{\KL}{D_{\mathrm{KL}}}

\usepackage{amsmath}
\usepackage{amssymb}
\usepackage{amsthm}
\usepackage{booktabs}
\usepackage{multirow}
\usepackage{graphicx}
\usepackage{wrapfig}
\providecommand{\WFclear}{}
\usepackage{float}
\usepackage{xcolor}
\usepackage{colortbl}
\definecolor{ourrow}{RGB}{226,236,244}
\definecolor{zerorow}{gray}{0.9}
\usepackage{algorithm}
\usepackage{algpseudocode}

\newcommand{\sil}{\varnothing}
\newcommand{\silencetok}{\texttt{<|silence|>}}

\newtheorem{proposition}{Proposition}[section]
\newtheorem{corollary}[proposition]{Corollary}
\iclrfinalcopy
\title{Live-OPSD: On-Policy Self-Distillation for Proactive Streaming Large Video Models}

\author{
    Kaixuan Du$^{1, *, \dagger}$ \quad 
    Xin Wan$^{1, *, \ddagger}$ \quad 
    Dai Guan$^{1}$ \quad 
    Ming Chen$^{1}$ \quad 
    Xiaoxi Jiang$^{1}$ \quad 
    Guanjun Jiang$^{1}$ \\[0.5em]
    $^1$Qwen Business Unit of Alibaba \\
    $^{*}$Equal contribution. \quad 
    $^{\dagger}$Work done during an internship at Alibaba. \quad
    $^{\ddagger}$Corresponding author.
}

\begin{document}
\maketitle

\begin{abstract}
Streaming multimodal assistants must make response decisions while events are
still unfolding. They need to recognize when the available evidence warrants a
response, produce an appropriate continuation, and remain silent otherwise.
Supervised fine-tuning learns from recorded assistant turns, but does not
directly train on the response prefixes produced by the evolving policy itself.
We introduce Live-OPSD, an on-policy self-distillation method that uses
privileged hints available only during training to supervise the policy on its own
outputs. The
policy first generates a continuation from the streaming context alone. The same
model then re-evaluates that continuation with a hint describing the relevant
evidence, and the resulting distributions guide the unhinted policy. On
Qwen3-Omni-30B-A3B, Live-OPSD attains an average score of 59.64 on LiveProBench,
5.26 points above the strongest continued-SFT baseline, and stays competitive on
OVO-Bench (65.84). These results indicate that guiding the policy on its own
outputs with privileged hints improves both when and how a streaming
assistant responds.
\end{abstract}

\section{Introduction}

For a streaming assistant, a correct answer can still be unhelpful if it arrives
at the wrong moment. Unlike offline video question answering, continuous
interaction requires decisions while evidence is still arriving. A user may
refer to an event that has just occurred or issue a standing instruction that
becomes actionable later. Speaking too early interrupts the interaction;
remaining silent after the relevant event misses the opportunity to help.
Proactive streaming therefore couples multimodal understanding with a temporal
decision: does the current evidence warrant a response?

Recent video-language models have moved toward continuous interaction by
combining streaming perception and temporally aligned dialogue with explicit
response control
\citep{chen2024videollmonline,wang2024mmduet,qian2025dispider,
wang2025streambridge,ding2025streammind,azad2026streamready}. Omni-modal
backbones further bring visual evidence, sound, and language into a shared model
\citep{xu2025qwen25omni,xu2025qwen3omni}. This progress provides the perception
and interaction interface needed for continuous assistance, but leaves an
important training question: how can a model learn from response prefixes
produced by its current policy rather than only from recorded assistant turns?

Supervised fine-tuning (SFT) optimizes reference continuations at annotated
prefixes. It does not directly revisit the prefixes produced when the policy
speaks too early, remains silent when a response is due, or begins an unsuitable
continuation despite making the correct speak-or-silence decision. Spoken
responses can also contribute many more training tokens than silence, creating
asymmetric supervision. AURA and StreamPro address this imbalance by reweighting
recorded supervision \citep{lu2026aura,li2026streampro}, but their targets remain
tied to annotated actions and prefixes. As the policy changes, the supervision
needed to correct its own continuations changes with it.

Live-OPSD targets this response-level gap through privileged-context
self-distillation (Figure~\ref{fig:framework}). At candidate decision points, the
current policy first generates a continuation from the streaming context alone.
The same model, now conditioned on a training-time hint that describes the
relevant evidence and intended behavior, provides distributional targets for
selected prefixes of those continuations. Distilling this privileged view guides
the unhinted policy on text it produced itself, combining learning with
privileged information \citep{vapnik2009new} and on-policy distillation
\citep{agarwal2024gkd}. Not all decision points carry the same training signal:
current disagreements expose errors, while consistent responses and silence
preserve behavioral coverage. Decision-Aware Selection (DAS) uses this structure
to construct a bounded update set.

We evaluate Live-OPSD as a post-training objective on Qwen2.5-Omni-7B and
Qwen3-Omni-30B-A3B. Controlled comparisons use a common streaming interface and
three continued-SFT objectives. On the 30B backbone, Live-OPSD attains an average
score of 59.64 on LiveProBench, compared with 54.38 for the strongest
continued-SFT baseline, a gain of 5.26 points. Released proactive systems provide
a separate system-level comparison on OVO-Bench and OmniMMI proactive alerting.
Ablations
compare the divergence objectives and evaluate the overall selection scheme
against a fixed-position control. Together, the results provide evidence that
post-training on current-policy response prefixes improves proactive streaming
behavior beyond continued SFT in the evaluated setting.

\begin{figure}[t]
\centering
\includegraphics[width=\linewidth]{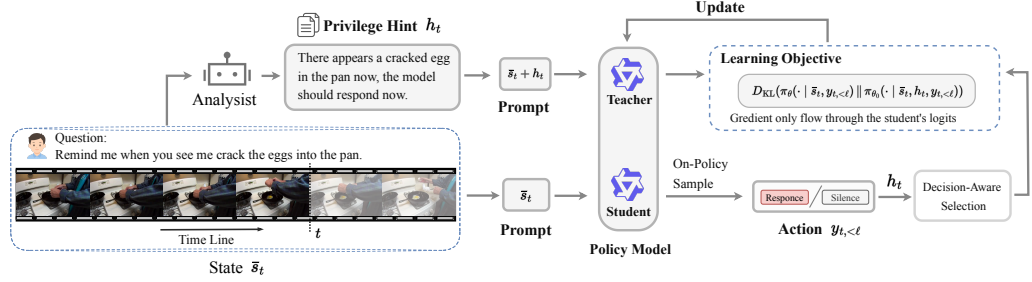}
\caption{Live-OPSD training. The unhinted policy generates a continuation, while
the same current model uses a privileged hint to score its prefixes. DAS selects
scored decision points for a reverse-KL update. The updated policy then supplies
the next rollout and teacher.}
\label{fig:framework}
\end{figure}

\section{Related Work}

\paragraph{Streaming perception and interaction.}
Streaming video models must preserve information from an evolving visual history
while remaining responsive to new observations and user turns. VideoLLM-online
learns temporally aligned dialogue over continuous video, whereas Flash-VStream
and related memory-based approaches compress or retrieve historical visual
states for efficient online question answering
\citep{chen2024videollmonline,zhang2024flashvstream}. MMDuet instead represents
video and text as interleaved turns and augments the model with informative and
relevance heads for time-sensitive interaction \citep{wang2024mmduet}. These
methods establish the streaming interface on which proactive behavior can be
built, but their primary contribution is not an on-policy post-training
objective.

Proactive systems differ in where they place the response decision. Dispider
separates perception, decision, and reaction; StreamBridge combines a persistent
history with a separately trained activation model; and StreamMind invokes the
language model only when a query-conditioned event gate predicts a
response-worthy moment
\citep{qian2025dispider,wang2025streambridge,ding2025streammind}. StreamReady
learns an explicit answer-readiness signal, Em-Garde matches structured visual
proposals during streaming, and StreamAgent uses future-aware planning to choose
whether to wait, respond, or gather additional evidence
\citep{azad2026streamready,zheng2026emgarde,yang2025streamagent，qian2025dispider}. These designs
modify inference-time gating, planning, memory, or computation. Live-OPSD instead
introduces a post-training objective for an existing streaming policy and leaves
its native inference mechanism unchanged.

Open-source proactive models provide further system-level comparisons. AURA
combines streaming context management with a Silent-Speech Balanced objective;
JoyAI-VL-Interaction makes response decisions within a vision-language model and
can delegate generation to a background model; ROMA uses aligned audio-visual
streaming units and a dedicated speak head; and MOSS-VL-Realtime permits incoming
visual information during generation through cross-attention
\citep{lu2026aura,yao2026joyaivlinteraction,tian2026roma,wang2026mossvl}. Their
architectural and training differences are retained in our model-level
comparison and are not conflated with the controlled objective comparison on two
Qwen backbones.

\paragraph{Learning response timing.}
Supervised timing objectives primarily differ in how they weight recorded
speech and silence. AURA masks supervised turns and reweights silent messages,
whereas StreamPro's CB-Stream Loss balances decision positions and response text
\citep{lu2026aura,li2026streampro}. StreamPro subsequently applies GRPO with
turn-level and trajectory-level rewards, and MMDuet2 studies multi-turn
reinforcement learning for timely proactive responses
\citep{li2026streampro,wang2025mmduet2}. These sequence-level objectives directly
optimize timing rewards but require long-stream rollouts and an explicit reward
definition. Live-OPSD uses a different decomposition: the current response
decision determines which contexts receive the update, while a privileged
teacher distribution supplies the token-level learning signal on the policy's
own continuation.

\paragraph{Distillation with privileged information.}
Knowledge distillation transfers a teacher's output distribution to a student
\citep{hinton2015distilling}. On-policy variants evaluate the teacher on
student-generated continuations, thereby exposing the objective to prefixes
visited by the current student \citep{agarwal2024gkd,gu2024minillm}. This shift is
related to the learner-induced state distribution studied in imitation learning
\citep{ross2011reduction}. Live-OPSD applies the same principle at streaming
decision contexts. Learning with privileged information supplies the asymmetry
\citep{vapnik2009new}: teacher and student share the current parameters, but only
the teacher receives a training-time hint.

\paragraph{Evaluation of streaming behavior.}
OVO-Bench and OmniMMI evaluate complementary aspects of causal stream
understanding, including real-time perception, contextual reasoning, forward
active responding, proactive alerting, and interaction
\citep{li2025ovobench,omnimmi}. StreamArena extends this scope to
hour-scale open-ended streams and jointly evaluates real-time perception,
historical recall, proactive interaction, and tool use \citep{zhang2026streamarena}.
We use OVO-Bench and OmniMMI to test whether timing-oriented post-training
preserves broad streaming capabilities, while LiveProBench isolates response
timing under standing instructions
\citep{du2026liveprobenchstreamingvideomodels}.

\section{Streaming Interaction}
\label{sec:formulation}

\paragraph{States and actions.}
Let $x_t$ be the event arriving at step $t$, either an audio-visual stream chunk
or a user utterance. The visible interaction history is
$s_t=(x_1,a_1,\ldots,a_{t-1},x_t)$, where $a_t$ denotes the assistant's complete
turn. The action space contains a set of natural-language responses
$\mathcal{A}_{\mathrm{speak}}$ and a silence action $\sil$.
The arriving event induces two state types and their admissible actions:
\begin{equation}
  \mathcal{A}(s_t)=
  \begin{cases}
    \mathcal{A}_{\mathrm{speak}}, & x_t\text{ is a user utterance},\\[2pt]
    \mathcal{A}_{\mathrm{speak}}\cup\{\sil\}, & x_t\text{ is a stream chunk}.
  \end{cases}
  \label{eq:admissible}
\end{equation}
A user utterance requires a response, which can include an acknowledgment or an
explanation that the available evidence is insufficient. A stream chunk permits
either response or silence. Equation~\ref{eq:admissible} specifies the task
requirements, although the language-model decoder can still emit an inadmissible
action.
After acting, the state advances to $s_{t+1}=s_t\oplus a_t\oplus x_{t+1}$.
The policy has access to the visible prefix, not future observations.

\paragraph{Decision supervision.}
Write $d(a)=\mathbb{1}[a\ne\sil]$ for the response decision induced by an action.
This decision is defined by the action's meaning under the interaction protocol,
independently of the tokens used to represent response or silence.
A training point consists of a recorded context $\bar{s}_t$, a supervisory
decision $b_t\in\{0,1\}$, and a privileged hint $h_t$. The decision specifies
whether the annotated behavior is to speak. It does not specify the full
response. At a chunk step, $b_t=1$ marks a proactive response and $b_t=0$ marks
silence. The hint supplies a short description of the relevant visual or acoustic
evidence and its relation to the intended behavior.
Figure~\ref{fig:data} shows one annotated dialogue with these units in place.

\begin{figure}[t]
\centering
\includegraphics[width=\linewidth]{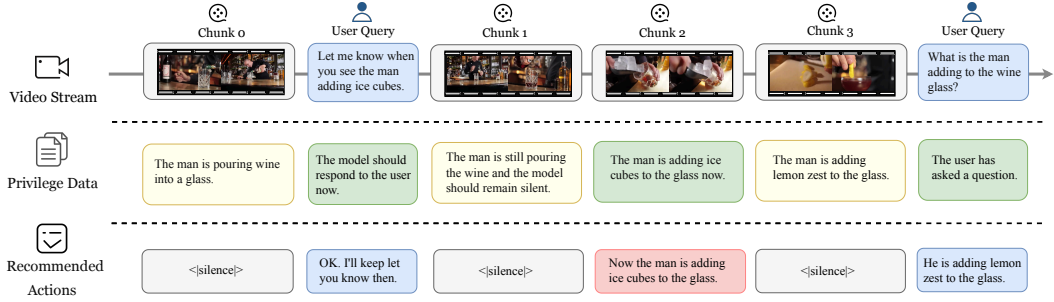}
\caption{Supervision at streaming decision points. Each interleaved chunk or user
turn carries a decision $b_t$ and a privileged hint $h_t$. Streaming SFT learns
the recorded action, whereas Live-OPSD supplies the hint only to the teacher.}
\label{fig:data}
\end{figure}

\paragraph{Training instantiation.}
Streaming SFT initializes the policy used in our experiments, including its
handling of interleaved chunks, dialogue turns, and silence. Its target at each
step is the recorded action associated with $b_t$: an annotated response when
$b_t=1$ and the silence action when $b_t=0$. Although $b_t$ specifies only the
binary decision, the recorded action supplies the full token-level target.
Writing $\bar a_t=(\bar y_{t,1},\ldots,\bar y_{t,|\bar a_t|})$ for this action,
let $\mathcal S$ contain the silent assistant messages in a training instance
and let $t^+$ denote its final non-silent assistant message, when one is present.
We supervise every token of messages in $\mathcal S\cup\{t^+\}$ and exclude
earlier non-silent messages whose supporting evidence may have been removed by
window truncation. Accordingly,
\begin{equation}
  m_{t,\ell}=\mathbb{1}[t\in\mathcal S\cup\{t^+\}],\qquad
  w_{t,\ell}=\begin{cases}
    1/N_{\mathrm{sil}}, & t\in\mathcal S,\\
    1, & t=t^+,
  \end{cases}
  \label{eq:sft-mask}
\end{equation}
where $N_{\mathrm{sil}}=|\mathcal S|$ is the number of silent messages in the
instance. With $N_{\mathrm{sup}}=\sum_t\sum_{\ell=1}^{|\bar a_t|}m_{t,\ell}$,
the initialization objective is
\begin{equation}
  \mathcal{L}_{\mathrm{SFT}}(\phi)
    =-\frac{1}{N_{\mathrm{sup}}}
      \sum_t\sum_{\ell=1}^{|\bar a_t|}m_{t,\ell}w_{t,\ell}
      \log\pi_\phi\!\left(\bar y_{t,\ell}\mid
      \bar s_t,\bar y_{t,<\ell}\right).
  \label{eq:sft}
\end{equation}
Thus, silence supervision is preserved without allowing the number of silent
messages to dominate the informative response. The checkpoint obtained by
minimizing $\mathcal{L}_{\mathrm{SFT}}(\phi)$, denoted
\emph{Live-SFT-Zero}, initializes Live-OPSD and supplies the first reference
parameters $\theta_0$. The SSB-SFT row in Section~\ref{sec:main-results}
continues this objective from the same initialization, whereas Vanilla-SFT and
CB-SFT use the alternative objectives defined in Appendix~\ref{app:baselines}.
The hint $h_t$ is not part of this stage's input or target.

Live-OPSD then generates a new assistant continuation at each recorded context
$\bar{s}_t$ under the current decoding rule, while retaining earlier assistant
turns from the recorded history. The resulting action, including a silence
action, and all of its token prefixes therefore come from the current policy.
This response-level on-policy rollout exposes the update to current-model
prefixes without claiming that the entire preceding interaction was generated
on policy. At deployment, earlier generated turns become part of the evolving
history $s_t$. Training and deployment thus differ in the source of earlier
assistant turns, while the causal observation boundary in
Equation~\ref{eq:admissible} applies in both.

\section{Live-OPSD}
\label{sec:method}

Live-OPSD couples privileged-context distribution matching with adaptive
selection of streaming decision points. The two supervisory inputs play distinct
roles: the hint $h_t$ conditions the teacher, and the decision $b_t$ guides DAS.
Figure~\ref{fig:framework} shows how these paths meet in the update objective.

\subsection{Privileged-context self-distillation}
\label{sec:teacher}

A hint explains the evidence supporting an annotated response decision. For a
request to report entry into a restricted area, for example, it might state that
a person is visible outside the area and the assistant should remain silent.
The teacher receives the hint in its system context. The student receives the
same streaming observations and deployment instruction without the hint.
In our implementation, an external model annotates hints offline. It supplies
text, whereas the distillation teacher is the current policy conditioned on
that text. Appendix~\ref{app:hints} gives the annotation and prompt format.

Let $\theta_0$ denote the parameters at the start of an update. For each point,
the student generates $a_t=(y_{t,1},\ldots,y_{t,L_t})$ from
$\pi_{\theta_0}(\cdot\mid\bar{s}_t)$ under the decoding rule used for training.
At response position $\ell$, we compare
\begin{equation}
  \begin{aligned}
    P_t^{(\ell)}
      &=\pi_{\theta_0}(\cdot\mid\bar{s}_t,h_t,y_{t,<\ell}),\\
    Q_{\theta,t}^{(\ell)}
      &=\pi_\theta(\cdot\mid\bar{s}_t,y_{t,<\ell}),
  \end{aligned}
  \label{eq:teacher-student}
\end{equation}
where $y_{t,<\ell}$ is the student-generated prefix. Both roles predict the next
token after that same prefix, so teacher scoring requires no separately generated
teacher response. The target and generated continuation remain fixed while the
student is updated from $\theta_0$.

After the update, the new parameters are used for both rollout and teacher
scoring. The offline hints are reused, but the distributional targets are
recomputed. The teacher's additional context can make the evidence--decision
relationship explicit, with its usefulness depending on annotation quality and
the model's ability to interpret the hint. Appendix~\ref{app:information}
distinguishes the information supplied by a hint from the decision advantage
actually realized by the hinted model.

\subsection{A distributional objective}
\label{sec:truncated}

For teacher distribution $P$ and student distribution $Q$, our default objective
is reverse KL, $\mathcal D_{\mathrm{rKL}}(P,Q)=\KL(Q\Vert P)$. We compare it
with forward KL, $\mathcal D_{\mathrm{fKL}}(P,Q)=\KL(P\Vert Q)$, and symmetric
Jensen--Shannon divergence. For completeness, the implemented skew-JSD family is
\begin{equation}
  D_{\mathrm{JS}}^{\beta}(P,Q)
    =\beta\,\KL(P\Vert M_\beta)+(1-\beta)\,\KL(Q\Vert M_\beta),
  \qquad M_\beta=\beta P+(1-\beta)Q,
  \label{eq:jsd}
\end{equation}
where $0<\beta<1$ and $\beta=\tfrac12$ gives symmetric JSD. The unscaled
divergence is zero at both endpoints. Its scaled limits are
\[
  \lim_{\beta\to0}\frac{D_{\mathrm{JS}}^{\beta}(P,Q)}{\beta}
    =\KL(P\Vert Q),\qquad
  \lim_{\beta\to1}\frac{D_{\mathrm{JS}}^{\beta}(P,Q)}{1-\beta}
    =\KL(Q\Vert P).
\]
We therefore implement the two KL directions as separate objectives rather than
as JSD endpoint values. In every case, gradients flow through $Q$ while $P$
remains fixed.

Let $\mathcal{B}_{\mathrm{sel}}$ be the effective set of selected, teacher-scored
points defined in Section~\ref{sec:selection}. For a chosen divergence
$\mathcal D$, the update minimizes its token mean over the generated assistant
continuations, including silence actions:
\begin{equation}
  \mathcal{L}_{\mathcal D}(\theta;\theta_0)
    =\frac{1}{\sum_{t\in\mathcal{B}_{\mathrm{sel}}}L_t}
      \sum_{t\in\mathcal{B}_{\mathrm{sel}}}\sum_{\ell=1}^{L_t}
      \mathcal D\!\left(P_t^{(\ell)},Q_{\theta,t}^{(\ell)}\right).
  \label{eq:loss}
\end{equation}
We use $\mathcal D=\mathcal D_{\mathrm{rKL}}$ by default. The generated
continuations, selected points, and teacher targets are fixed during this update.
Gradients optimize the student's conditional distributions on these prefixes
without differentiating through the discrete decoding process. Thus,
current-policy generation supplies the training prefixes, while distillation
supplies the update. Each point contributes in proportion to its continuation
length under the global token-mean reduction. Appendix~\ref{app:impl} describes
the efficient target representation used in the implementation.

Appendix~\ref{app:decision-transfer} connects the default reverse-KL objective to
response decisions. It bounds the student's disagreement by the teacher's
disagreement plus cumulative reverse KL on the student-generated prefix that
resolves the decision and characterizes the approximation used in training.

\subsection{Decision-Aware Selection}
\label{sec:selection}

DAS uses the current policy's response decisions to allocate a budget of at most
$K$ training points per update. Eligible points are chunk steps, with user
utterances and their responses kept in context. For a generated action $a_t$, the
indicator $\delta_t=\mathbb{1}[d(a_t)\ne b_t]$ records disagreement with the
supervisory decision $b_t$. Eligible points split into decision-inconsistent
points $\mathcal{I}$ ($\delta_t=1$), consistent responses $\mathcal{R}$, and
consistent silence; teacher scoring covers $\mathcal{I}$, $\mathcal{R}$, and a
policy-independent subsample of consistent silence.

From this scored pool DAS keeps at most $K$ points in priority order: an
inconsistency quota with the two directions alternated, then consistent
responses, then silence spread across videos, relaxing to the full budget when
scored silence is scarce. Selection, not gradient weighting, depends on
$\delta_t$: a decision-consistent action can still yield a distillation signal,
because agreeing on the decision does not imply agreeing on the token
distributions. Selection is refreshed from new rollouts each update. At
inference, the policy uses the backbone's native response interface without
hints, DAS, or an extra classifier. Figure~\ref{fig:das} shows the outcome
classes and how the scored pool fills the budget; Appendices~\ref{app:selection}
and~\ref{app:das-analysis} give the quota rule and the disagreement it
concentrates, and Algorithm~\ref{alg:main} (Appendix~\ref{app:procedure})
summarizes one update.

\begin{figure}[t]
\centering
\includegraphics[width=\linewidth]{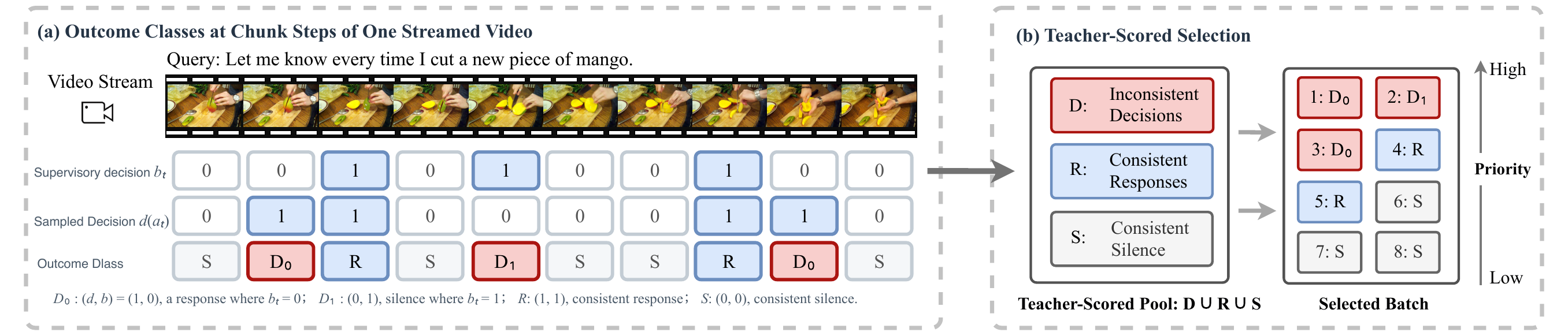}
\caption{Decision-Aware Selection. \textbf{(a)} Sampled decisions are grouped as
the two inconsistency directions, consistent responses, or consistent silence.
\textbf{(b)} DAS fills a budget of $K$ scored points from an inconsistency quota,
then consistent responses and silence. The decision label controls selection,
not gradient weighting.}
\label{fig:das}
\end{figure}

\section{Experiments}
\label{sec:experiments}

The evaluation examines whether post-training improves response timing while
preserving the understanding needed for useful responses. Within-backbone
comparisons assess the training procedure, system-level comparisons place it
alongside released proactive models, and ablations isolate the divergence and
selection rule.

\subsection{Experimental setup}
\label{sec:setup}

\paragraph{Backbones and training data.}

We use Qwen2.5-Omni-7B and Qwen3-Omni-30B-A3B
\citep{xu2025qwen25omni,xu2025qwen3omni}. Both are adapted to the same streaming
interaction protocol, with audio-visual chunks interleaved with user and assistant
turns. The training corpus includes direct questions and standing instructions
that require proactive responses. For post-training, dialogues are expanded into
causal decision contexts with aligned supervisory decisions and hints. Source
videos define the split boundary so that contexts from one source remain in one
split. Appendix~\ref{app:data} describes data construction, and
Appendix~\ref{app:impl} reports the run settings verified from the available
artifacts.

\paragraph{Supervised baselines.}
We compare three objectives on each backbone.
\emph{Vanilla-SFT} uses standard cross-entropy on reference assistant turns.
\emph{CB-SFT} uses StreamPro's CB-Stream Loss \citep{li2026streampro}.
\emph{SSB-SFT} uses AURA's Silent-Speech Balanced Loss\citep{lu2026aura}. Live-OPSD starts from a same streaming-SFT model and then applies the
post-training procedure in Section~\ref{sec:method}.

\paragraph{Open-source proactive models.}
The model-level comparison on LiveProBench (Table~\ref{tab:liveprobench}) includes
AURA, JoyAI-VL-Interaction, ROMA, and MOSS-VL-Realtime
\citep{lu2026aura,yao2026joyaivlinteraction,tian2026roma,wang2026mossvl}.
Each model retains the response mechanism and context handling
of its released inference configuration. For each benchmark, task-specific
scoring and aggregation follow the protocol described below and in
Appendix~\ref{app:metrics}.

\begin{table}[t]
\caption{LiveProBench results (\%). The average is computed from the unrounded
Event Sensitivity, Onset Response, Window Response, Anomaly Alert, Sustained
Silence, and Duplicate Counting scores. Boldface marks the best result within
each block, excluding Live-SFT-Zero.}
\label{tab:liveprobench}
\centering
\small
\resizebox{\textwidth}{!}{%
\begin{tabular}{lccccccccc}
\toprule
Method or model & \multicolumn{3}{c}{Event Sensitivity} & OR & WR & AA & SS & DC & Avg. \\
\cmidrule(lr){2-4}
 & Resp. & Sil. & Geo. & & & & & & \\
\midrule
\multicolumn{10}{l}{\emph{Open-source proactive models}}\\
ROMA                 & 32.87 & 61.79 & 45.07 & 19.27 & 6.25 & 9.24 & 2.17 & 2.99 & 14.16 \\
MOSS-VL-Realtime     & 82.87 & 50.41 & 64.63 & 38.07 & 35.00 & 24.79 & 19.57 & 7.46 & 31.59 \\
AURA                 & 34.26 & 99.19 & 58.29 & 62.84 & 60.00 & 37.82 & 25.00 & 0.00 & 40.66 \\
JoyAI-VL-Interaction & 47.22 & 78.86 & 61.02 & 52.29 & 47.50 & 18.91 & 30.43 & 35.82 & 41.00 \\
\midrule
\multicolumn{10}{l}{\emph{Qwen2.5-Omni-7B}}\\
\rowcolor{zerorow}Live-SFT-Zero-7B & 63.43 & 59.35 & 61.35 & 42.66 & 27.50 & 37.82 & 2.17 & 34.33 & 34.31 \\
\quad + Vanilla-SFT & 36.57 & \textbf{98.78} & 60.11 & 55.50 & \textbf{65.00} & 54.62 & \textbf{53.26} & 1.49 & 48.33 \\
\quad + CB-SFT   & 68.06 & 78.86 & \textbf{73.26} & \textbf{61.47} & 33.75 & 62.61 & 11.96 & \textbf{26.87} & 44.98 \\
\quad + SSB-SFT  & \textbf{75.93} & 64.23 & 69.83 & 55.05 & 17.50 & 49.58 & 3.26 & 7.46 & 33.78 \\
\rowcolor{ourrow}\quad + Live-OPSD & 41.20 & 95.12 & 62.60 & 59.63 & 58.75 & \textbf{65.13} & 40.22 & 7.46 & \textbf{48.97} \\
\midrule
\multicolumn{10}{l}{\emph{Qwen3-Omni-30B-A3B}}\\
\rowcolor{zerorow}Live-SFT-Zero-30B & 41.67 & 84.55 & 59.36 & 52.29 & 43.75 & 73.11 & 42.39 & 37.31 & 51.37 \\
\quad + Vanilla-SFT & 41.67 & 99.19 & 64.29 & 63.30 & 67.50 & 57.56 & \textbf{70.65} & 2.99 & 54.38 \\
\quad + CB-SFT   & 68.06 & 86.18 & \textbf{76.58} & \textbf{72.94} & 45.00 & 68.07 & 31.52 & 23.88 & 53.00 \\
\quad + SSB-SFT  & \textbf{72.69} & 78.05 & 75.32 & 64.68 & 41.25 & 65.55 & 23.91 & 32.84 & 50.59 \\
\rowcolor{ourrow}\quad + Live-OPSD & 40.28 & \textbf{99.59} & 63.34 & 61.93 & \textbf{70.00} & \textbf{68.49} & 59.78 & \textbf{34.33} & \textbf{59.64} \\
\bottomrule
\end{tabular}
}
\end{table}

\subsection{Benchmarks and metrics}
\label{sec:benchmarks}

We assess proactive response timing through event triggering and silence in
LiveProBench and proactive alerting (PA) in OmniMMI \citep{du2026liveprobenchstreamingvideomodels,omnimmi}.
OmniMMI's other tasks additionally assess the perceptual, contextual, and
interaction capabilities needed for useful responses.
OVO-Bench \citep{li2025ovobench} complements this coverage with understanding of
past and present evidence and forward active responding, which tests waiting for
the information needed to answer.

Table~\ref{tab:liveprobench} reports LiveProBench's six task scores and their
arithmetic mean, including Duplicate Counting. Event Sensitivity uses the
geometric mean of alert coverage and silent-segment accuracy, and we list both
sides alongside that mean; the window tasks score the first response within the
permitted interval, and Duplicate Counting uses exact-count accuracy. Task
coverage, temporal-window scoring, and full aggregation details are specified in
Appendix~\ref{app:metrics}.

\subsection{Main comparisons}
\label{sec:main-results}

\begin{wraptable}{r}{0.36\textwidth}
\vspace{-0.8em}
\caption{OmniMMI Proactive Alerting accuracy (\%).}
\label{tab:pa}
\centering
\footnotesize
\setlength{\tabcolsep}{3.5pt}
\begin{tabular*}{\linewidth}{@{\extracolsep{\fill}}lc@{}}
\toprule
Model & PA \\
\midrule
VideoLLM-Online & 0.50 \\
MMDuet                  & 22.00 \\
M4-a                  & 25.50 \\
AURA                & 37.50 \\
ROMA                & 37.50 \\
\midrule
Live-SFT-Zero-7B      & 31.50 \\
\quad + Live-OPSD      & 56.50 (+25.0) \\
Live-SFT-Zero-30B      & 39.50 \\
\quad + Live-OPSD     & \textbf{59.00 (+19.5)} \\
\bottomrule
\end{tabular*}
\end{wraptable}

Table~\ref{tab:liveprobench} is the primary comparison: it lists released
proactive models first and then contrasts training methods within each backbone,
including Live-SFT-Zero, the first-stage streaming-SFT checkpoint from which
Live-OPSD training starts. All backbone rows are evaluated with audio disabled,
and the reported 30B Live-OPSD model uses reverse KL. These matched-backbone blocks
isolate the effect of the training method from differences in model scale.
Tables~\ref{tab:pa} and~\ref{tab:ovo-bench} instead compare complete models on
OmniMMI PA and OVO-Bench, excluding the continued-SFT variants. Together, these
system-level comparisons complement the controlled LiveProBench results under
established external benchmark protocols. Full system-level comparisons and subtask
breakdowns for StreamingBench and OVO-Bench are provided in
Appendix~\ref{app:metrics}.

Response coverage and silence trade off: frequent responses cover more events but
interrupt, while frequent silence avoids interruptions but misses events. Event
Sensitivity scores both sides at once, so its geometric mean (Geo.) rises only
when response coverage (Resp.) and silence accuracy (Sil.) are both high; a
near-complete silence side with low coverage is not credited as balanced.

This trade-off is visible in the Qwen3-Omni-30B-A3B block of
Table~\ref{tab:liveprobench}. The reverse-KL Live-OPSD model has the highest
reported six-task average (59.64), exceeding Vanilla-SFT, CB-SFT, and SSB-SFT by
5.26, 6.64, and 9.05 points, respectively. The Live-SFT-Zero rows identify the
checkpoints used before Live-OPSD post-training. Appendix~\ref{app:case-study}
compares these methods on individual cases (Figure~\ref{fig:anomaly-cases}).

\WFclear
\begin{wraptable}{l}{0.36\textwidth}
\vspace{-0.8em}
\caption{OVO-Bench overall accuracy (\%) across released systems.}
\label{tab:ovo-bench}
\centering
\footnotesize
\setlength{\tabcolsep}{3.5pt}
\begin{tabular*}{\linewidth}{@{\extracolsep{\fill}}lc@{}}
\toprule
Model & All \\
\midrule
StreamAgent           & 49.40 \\
Streamo-7B            & 57.90 \\
ViSpeak               & 61.10 \\
MiniCPM-o-4.5         & 59.70 \\
AURA                  & 65.30 \\
\midrule
Qwen2.5-Omni-7B       & 50.30 \\
\quad + Live-SFT \& Live-OPSD     & 61.97 \\
Qwen3-Omni-30B-A3B    & 61.00 \\
\quad + Live-SFT \& Live-OPSD     & \textbf{65.84} \\
\bottomrule
\end{tabular*}
\end{wraptable}

The 7B comparison follows the same evaluation protocol. Live-OPSD obtains a
reported average of 48.97, compared with 48.33, 44.98, and 33.78 for Vanilla-SFT,
CB-SFT, and SSB-SFT, respectively. It leads Anomaly Alert, while CB-SFT leads
Event Sensitivity, Onset Response, and Duplicate Counting, and Vanilla-SFT leads
Window Response and Sustained Silence; Live-OPSD therefore has the highest point
estimate in both backbone blocks.

Table~\ref{tab:pa} foregrounds proactive response. The two Live-OPSD models
reach 56.50 and 59.00 on OmniMMI PA, compared with 37.50 for AURA. These published system-level results use different inference configurations;
controlled objective comparisons appear in Table~\ref{tab:liveprobench}. Table~\ref{tab:ovo-bench} reports overall OVO-Bench accuracy. Live-OPSD improves over
its backbone at both scales (50.30 to 61.97 at 7B, 61.00 to 65.84 at 30B) and
reaches 65.84 overall.

\subsection{Ablation studies}
\label{sec:ablation}

\paragraph{Divergence choice.}
\begin{wrapfigure}{r}{0.4\textwidth}
\centering
\includegraphics[width=0.4\textwidth]{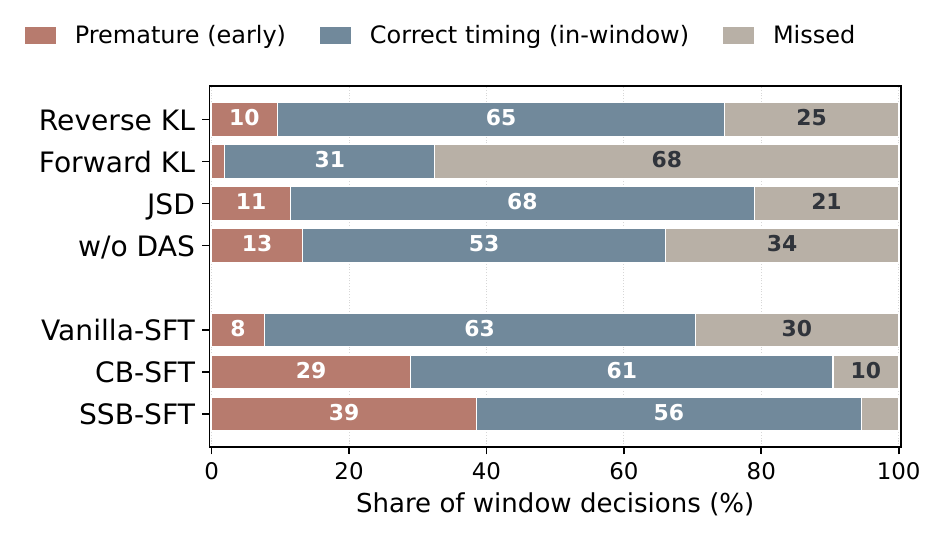}
\caption{LiveProBench timing outcomes for Qwen3-Omni.}
\label{fig:timing}
\end{wrapfigure}

We compare reverse KL, forward KL, and JSD under identical training setups.
For teacher $P$ and student $Q$, the two KL objectives are
$\KL(P\Vert Q)$ and $\KL(Q\Vert P)$, respectively. Their probability weighting
differs: the forward direction emphasizes teacher-supported alternatives, while
the reverse direction emphasizes alternatives currently favored by the student.
The comparison tests how this probability weighting affects learning in the
shared-parameter setting. The optimization budget and learning-rate selection protocol are held fixed. On
the shared held-out set, reverse KL leads proactive $F_1$, false-alarm rate, and
per-episode exact match for most of training, accepting a higher event-miss rate
as its conservative trade-off. Forward KL and the fixed-suffix w/o DAS variant
are less stable (Figure~\ref{fig:ablation-curves}, Appendix~\ref{app:metrics}). On
LiveProBench (Table~\ref{tab:ablation}), forward KL under-triggers and misses most
response windows (Figure~\ref{fig:timing}). Its event-sensitivity and
anomaly-alert scores fall to 43.03 and 11.34, producing a 36.10 mean. Reverse KL
and JSD are comparable in mean score (59.64 versus 59.21), while reverse KL is
higher on window response (70.00 versus 66.25) and duplicate counting (34.33
versus 22.39). We report reverse KL as the 30B Live-OPSD model in
Tables~\ref{tab:liveprobench}, \ref{tab:pa}, and~\ref{tab:ovo-bench}.

\begin{figure}[h]
\centering
\includegraphics[width=\textwidth]{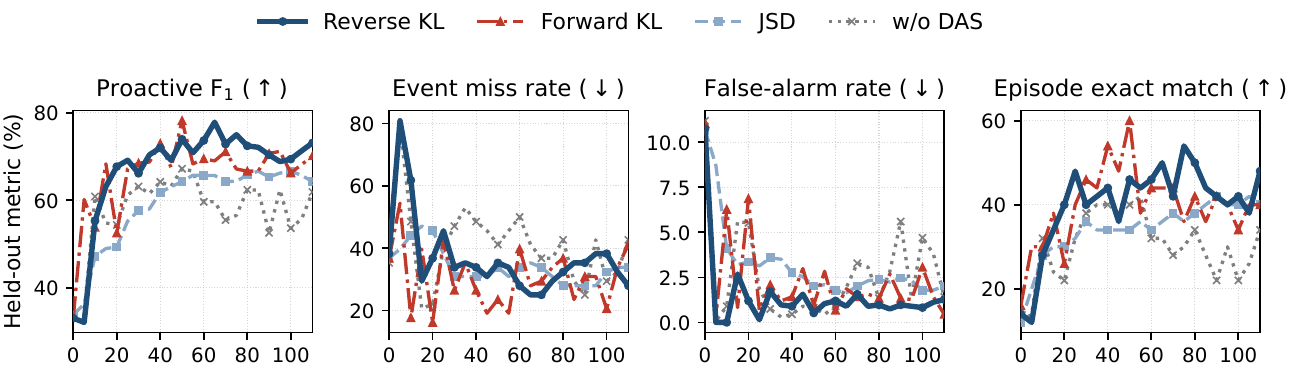}
\vspace{-0.45em}
\caption{Held-out ablation dynamics over 110 steps on 1,458 fixed decision
points; arrows indicate the improving direction.}
\label{fig:ablation-curves}
\end{figure}

\paragraph{Removing Decision-Aware Selection.}
The \emph{w/o DAS} variant uses a fixed-size suffix of decision points from each
source video, chosen before training independently of the policy. Each point
retains its complete preceding dialogue and media context. Teacher hints and the inference protocol are unchanged. The selected-point budget and update
count are matched, with candidate-generation and teacher-scoring costs reported
separately. This comparison tests adaptive allocation against fixed-position
selection. Appendix~\ref{app:selection} specifies
the fixed-position rule and its distinction from history truncation. Replacing DAS with fixed-position selection lowers the LiveProBench mean from
59.21 to 44.18 (Table~\ref{tab:ablation}). Duplicate Counting falls to 0.00, and
both premature and missed timing bands widen relative to the full JSD variant
(Figure~\ref{fig:timing}).

\begin{table}[t]
\caption{LiveProBench ablations on Qwen3-Omni (\%).
The main Live-OPSD result uses reverse KL. The w/o DAS control retains JSD and
replaces adaptive selection with a fixed-position rule.}
\label{tab:ablation}
\centering
\small
\begin{tabular*}{\textwidth}{@{\extracolsep{\fill}}lccccccc}
\toprule
Variant & ES & OR & WR & AA & SS & DC & Avg. \\
\midrule
Reverse KL & 63.34 & 61.93 & \textbf{70.00} & 68.49 & 59.78 & \textbf{34.33} & \textbf{59.64} \\
Forward KL & 43.03 & 31.19 & 58.75 & 11.34 & 54.35 & 17.91 & 36.10 \\
JSD & \textbf{67.08} & \textbf{67.43} & 66.25 & \textbf{70.17} & \textbf{61.96} & 22.39 & 59.21 \\
w/o DAS & 59.58 & 55.50 & 51.25 & 54.20 & 44.57 & 0.00 & 44.18 \\
\bottomrule
\end{tabular*}
\end{table}

\section{Discussion and Limitations}
\label{sec:discussion}

\paragraph{Scope and transfer requirements.}
Live-OPSD is a post-training principle for streaming interaction, instantiated with a streaming SFT checkpoint and a specific protocol. Defined through conditional policies and semantic actions rather than a mandatory two-stage framework, it requires aligned decision supervision and conditional likelihood evaluation with and without privileged context for protocol adaptation. For example, JoyAI-VL-Interaction marks silence and response with dedicated tokens~\citep{yao2026joyaivlinteraction}; the decision map $d$ reads those markers while preserving the native protocol, with teacher and student sharing a vocabulary within each backbone. Richer action spaces, such as introducing interrupt and correction actions, require corresponding supervision and selection categories.

\paragraph{Current instantiation.}
Offline hints and recorded histories make training tractable. Pilot checks found Qwen3-Omni inconsistent as its own hint annotator over multi-turn histories, so we use an external model for annotation and reserve the hinted policy for distributional scoring—motivating an implementation choice rather than revealing a backbone limitation. Generating continuations at recorded histories provides response-level on-policy supervision but mismatches model-generated deployment histories. Thus, the source of histories and feedback is an implementation choice, whereas access only to observations at each decision is a causal constraint.

\paragraph{Toward online self-training.}
Online hint generation and trajectory-level on-policy generation would align
feedback with student-generated histories, and reliable multi-turn
self-assessment could eventually let the model propose its own hints, provided
they stay grounded in the observed stream and add information beyond its
unassisted judgment. Because the current method still depends on external hints
and supervised response decisions, these are steps toward online self-training
rather than a fully self-supervised procedure.

\section{Conclusion}

We presented Live-OPSD, a post-training method that uses privileged context to supervise a streaming policy on its own continuations. A hinted model copy supplies token-level targets, while Decision-Aware Selection allocates the update budget by response decisions. On Qwen3-Omni, Live-OPSD scores 59.64 on LiveProBench, 5.26 points above the strongest continued-SFT baseline, while remaining competitive on broader benchmarks. Live-OPSD thus provides evidence for privileged on-policy distillation as a promising training principle for proactive streaming video models.

\subsubsection*{AI Usage Statement}
Generative AI tools were used only for language polishing during manuscript preparation. They were also used to synthesize SFT data and offline privileged hints as training supervision; the hint production procedure is described in Appendix~\ref{app:hints}. They were not used to execute training or evaluation runs or to generate or alter the reported experimental measurements. The authors take full responsibility for the submitted content.

\subsubsection*{Ethics Statement}
All videos used in this study are drawn from publicly available datasets, and no additional private video data are introduced. The study is limited to model training and benchmark evaluation. Potential biases and privacy considerations inherited from the source datasets and pretrained models remain applicable.

\subsubsection*{Reproducibility Statement}
Sections~\ref{sec:method} and~\ref{sec:setup} specify the method and comparison
protocol. The appendices describe data construction, baseline objectives,
selection, metrics, and settings verified from the available run artifacts. To avoid introducing unverified details, we report only configurations supported by the retained experimental records.

\bibliography{live_opsd}
\bibliographystyle{iclr2027_conference}

\appendix

\section{Training data and streaming contexts}
\label{app:data}

Each dialogue interleaves audio-visual chunks with user instructions and
assistant actions. Direct questions supervise responses from evidence available
at the question time. Standing instructions instead specify a condition to
monitor, with silence between relevant events. A decision context includes all
preceding turns and the media visible at that point. Flattening a dialogue into
training points changes the units sampled for optimization, not the information
available within each context.

For Live-OPSD, each retained point has a binary supervisory decision and a
teacher hint. Points without usable teacher supervision are excluded from the
effective distillation set. Data statistics distinguish source videos, dialogues, and decision points
because multiple contexts can come from one source. The split is established before
expansion, and training sources are checked against benchmark sources to avoid
overlap. Complete corpus counts are not available in the current artifact and are
therefore not reported here.

The Live-OPSD dataset contains 3,712 training instances. Both Live-OPSD and the continued-SFT baselines use this dataset, with the latter trained on its reference responses. The preceding streaming-SFT stage uses approximately 196,000 training instances. All source videos are drawn from publicly available datasets.

\section{Privileged hints and prompts}
\label{app:hints}

An external annotation model produces hints offline from the dialogue, reference
decision, and source media. The annotator can access the full clip, but the hint
concerns the evidence relevant to the current decision. It states what is visible
or audible, whether the assistant should respond, and the reason. Because it is
analysis rather than the answer, the reference response never appears inside the
hint. The hint is passed to the current policy as privileged context; this
hinted policy supplies
the teacher distributions in Equation~\ref{eq:teacher-student}. The student's
prompt and multimodal input remain bounded by the current streaming time.

The deployed system prompt is:
\begin{quote}\small
You are an AI assistant that processes real-time streaming video. For each video
stream chunk in every interaction turn, decide whether to remain silent or
respond. \\
- If you remain silent, output only \silencetok. \\
- If you respond, do so
naturally at the appropriate moment.
\end{quote}
The teacher and student receive the same multimodal input and the same
interaction to be answered, while the teacher additionally receives the
point-specific hint in current stream context. They share the same chat
template, tokenizer, and media representation. Their prompt lengths can differ,
so response positions are aligned relative to the end of each prompt.

For a standing instruction to report a person in a white shirt entering an area,
an illustrative hint is:
\begin{quote}\small
A person in a dark shirt is visible outside the monitored area, but no person in
a white shirt has entered it, so the assistant should remain silent.
\end{quote}
The hint is a teacher input, not a reference sentence for the student to repeat.
The current artifact does not preserve the annotation-model identity or retention
counts required to regenerate these hints exactly.

\section{Selection details and fixed-selection control}
\label{app:selection}

Let $\mathcal{I}$, $\mathcal{R}$, and $\mathcal{S}$ denote the scored pools of
decision-inconsistent points, consistent responses, and consistent silence.
Write $\mathcal{I}_0=\{t\in\mathcal{I}:b_t=0\}$ and
$\mathcal{I}_1=\{t\in\mathcal{I}:b_t=1\}$ for the two inconsistency directions.
Their orders are randomized independently and interleaved until one pool is
exhausted, after which the remaining points of the other pool follow.

A fraction $\alpha\in(0,1]$ sets the nominal discrepancy quota
$q=\max\{1,\lfloor\alpha K\rfloor\}$.
If the scored consistent-silence pool has fewer than $q$ points, the quota is
relaxed to $K$. Otherwise, at most $q$ inconsistent points enter the ordered
selection. The remaining slots are filled first from randomized consistent
responses, then from consistent-silence points distributed round-robin across
videos, with randomized order inside each video. Selection affects the sample
composition; it does not introduce inverse-frequency weights into
Equation~\ref{eq:loss}.

The consistent-silence scoring subset is determined independently of the current
response probabilities, using a stable point-level subsampling rule. The
selected effective set contains only scored points. The scoring fraction, quota,
and effective number of contributing points jointly determine the composition
and cost of an update; the current run artifacts retain the selected budget but
not the other two values.

For \emph{w/o DAS}, each source video contributes a fixed-size suffix of its
decision points. This set is chosen before training and is unchanged by the
current policy. Earlier dialogue turns and media remain in the context of every
retained point. The available run artifacts identify the 128-point selected batch but not the
per-video suffix length. All retained eligible points contribute without
decision-dependent prioritization or decision-dependent teacher admission.
Validation and test streams retain their full decision schedules. Matching the
update budget does not imply equal total computation: the full method also
generates candidates that may not contribute to an update.

\section{Training procedure}
\label{app:procedure}

Algorithm~\ref{alg:main} gives one Live-OPSD update; rollout, teacher scoring,
and selection are refreshed at each iteration.

\begin{algorithm}[h]
\caption{One Live-OPSD update with Decision-Aware Selection}
\label{alg:main}
\begin{algorithmic}[1]
\Require Current parameters $\theta_0$; eligible points
  $\{(\bar{s}_t,h_t,b_t)\}_{t\in\mathcal{B}}$;
  budget $K$, quota fraction $\alpha$, step size $\eta$
\Ensure Updated student parameters $\theta$
\For{$t\in\mathcal{B}$}
  \State $a_t\gets\operatorname{Generate}(\pi_{\theta_0},\bar{s}_t)$
  \State $\delta_t\gets\mathbb{1}[d(a_t)\ne b_t]$
\EndFor
\State $\mathcal{I}\gets\{t:\delta_t=1\}$;
  $\mathcal{R}\gets\{t:\delta_t=0,\ b_t=1\}$
\State $\mathcal{S}\gets$ a policy-independent subsample of
  $\{t:\delta_t=0,\ b_t=0\}$
\State $\mathcal{B}_{\mathrm{score}}\gets\mathcal{I}\cup\mathcal{R}\cup\mathcal{S}$
\For{$t\in\mathcal{B}_{\mathrm{score}}$}
  \For{$\ell=1,\ldots,L_t$}
    \State $P_t^{(\ell)}\gets\pi_{\theta_0}(\cdot\mid\bar{s}_t,h_t,y_{t,<\ell})$
  \EndFor
\EndFor
\State $q\gets\max\{1,\lfloor\alpha K\rfloor\}$
\If{$|\mathcal{S}|<q$}
  \State $q\gets K$
\EndIf
\State Order $\mathcal{I}$ by alternating its two inconsistency directions;
  randomize within each direction
\State Retain at most $q$ points from $\mathcal{I}$; randomize $\mathcal{R}$;
  order $\mathcal{S}$ across videos
\State $\mathcal{B}_{\mathrm{sel}}\gets$ up to $K$ points from the ordered groups
  $\mathcal{I}$, $\mathcal{R}$, $\mathcal{S}$ in succession
\For{$t\in\mathcal{B}_{\mathrm{sel}}$}
  \For{$\ell=1,\ldots,L_t$}
    \State $Q_{\theta,t}^{(\ell)}\gets\pi_\theta(\cdot\mid\bar{s}_t,y_{t,<\ell})$
  \EndFor
\EndFor
\State $\theta\gets\theta_0-\eta\nabla_\theta\mathcal{L}_{\mathcal D}(\theta;\theta_0)
  \big|_{\theta=\theta_0}$
\end{algorithmic}
\end{algorithm}

\section{Supervised baseline objectives}
\label{app:baselines}

The supervised baselines share the streaming representation, reference assistant
continuations, initialization checkpoint, and training schedule. They differ
only in which assistant tokens are supervised and how those tokens are weighted.
For a target token $\bar y_{t,\ell}$, write
\begin{equation}
  \ell_{t,\ell}(\phi)
    =-\log\pi_\phi\!\left(\bar y_{t,\ell}\mid
      \bar s_t,\bar y_{t,<\ell}\right).
  \label{eq:token-ce}
\end{equation}

\paragraph{Vanilla-SFT.}
Let $\mathcal A$ be the set of reference assistant messages retained in the
training instances and let
$N_{\mathrm{all}}=\sum_{t\in\mathcal A}|\bar a_t|$. Vanilla-SFT supervises all
of their target tokens with equal weight:
\begin{equation}
  \mathcal L_{\mathrm{vanilla}}(\phi)
    =\frac{1}{N_{\mathrm{all}}}
      \sum_{t\in\mathcal A}\sum_{\ell=1}^{|\bar a_t|}
      \ell_{t,\ell}(\phi).
  \label{eq:vanilla-sft}
\end{equation}

\paragraph{Silent-Speech Balanced SFT.}
SSB-SFT uses the supervision mask and message-level weighting in
Equation~\ref{eq:sft-mask}, following the silent-speech balanced objective of
AURA \citep{lu2026aura}. Within each training instance, all silent assistant
messages and the final non-silent assistant message are retained. Earlier
non-silent messages are masked out because a truncated window need not retain
the evidence aligned with those responses. If the instance contains
$N_{\mathrm{sil}}$ silent messages, every token in each silent message receives
weight $1/N_{\mathrm{sil}}$, while every token in the retained non-silent message
receives unit weight. With $N_{\mathrm{sup}}$ denoting the unweighted number of
retained target tokens, the resulting loss is exactly
Equation~\ref{eq:sft}. Here $N_{\mathrm{sil}}$ counts messages rather than
tokens and is computed per training instance. Live-SFT-Zero is trained with
this objective; the SSB-SFT baseline continues the same supervised objective
from that checkpoint for the post-training comparison.

\paragraph{Class-Balanced SFT.}
CB-SFT transfers the supervised CB-Stream objective \citep{li2026streampro} to
our two omni-modal backbones. Let $\mathcal D$ denote the target positions at
which the model chooses between the silence and response decision tokens, and
let $\mathcal T$ contain the subsequent response-text positions. For decision
class $c\in\{0,1\}$, let $n_c$ be its count in the current batch. The effective
sample count and normalized class weight are
\begin{equation}
  E_c=\frac{1-\beta^{n_c}}{1-\beta},\qquad
  \alpha_c=\frac{E_c^{-1}}{\sum_{j\in\{0,1\}}E_j^{-1}}\,,
  \label{eq:cb-weights}
\end{equation}
where $\beta\in[0,1)$ controls the strength of reweighting. For each target
position $i$, define
\begin{equation}
  w_i^{\mathrm{CB}}=\begin{cases}
    \alpha_{c_i}, & i\in\mathcal D,\\
    \lambda_{\mathrm{text}}, & i\in\mathcal T,
  \end{cases}
  \qquad
  \mathcal L_{\mathrm{CB}}(\phi)
    =\frac{1}{N_{\mathrm{CB}}}\sum_{i\in\mathcal D\cup\mathcal T}
      w_i^{\mathrm{CB}}\ell_i(\phi),
  \label{eq:cb-sft}
\end{equation}
where $c_i$ is the silence or response class at decision position $i$,
$\lambda_{\mathrm{text}}$ controls the contribution of linguistic tokens, and
$N_{\mathrm{CB}}=|\mathcal D|+|\mathcal T|$. This baseline contains only the
supervised objective, not StreamPro's subsequent policy-optimization stage. The
available run artifacts do not preserve the baseline-specific values of
$\beta$ and $\lambda_{\mathrm{text}}$, so we do not infer them from another
implementation.

The CB-SFT and SSB-SFT names therefore identify controlled objective variants.
The AURA entry in Tables~\ref{tab:liveprobench} and \ref{tab:pa} instead refers to the released complete system.

\section{Experimental configuration}
\label{app:impl}

\paragraph{Data format and preprocessing.}
Each stream is segmented into one-second chunks. A chunk is represented by two
frames, i.e.\ a 2\,fps visual rate, together with a 16\,kHz mono audio segment;
every frame is rescaled to at most $448\times448$ pixels ($200{,}704$). Chunks are
encoded as separate image and audio inputs rather than as audio-in-video, so a
multi-chunk stream maps to an unambiguous token sequence. Standing instructions,
user turns, and assistant decisions are interleaved with these chunk tokens in
streaming order, and the same representation is used for teacher scoring and
student rollout. The context window is $65{,}536$ tokens: Qwen3-Omni-30B-A3B uses
its native 64k context, and the smaller Qwen2.5-Omni backbone (native 32k) is
extended to the same window with YaRN \citep{peng2024yarn}.

\begin{table}[h]
\caption{Verified configuration of the Qwen3-Omni-30B-A3B reverse-KL run used in
the main result and divergence ablation.}
\label{tab:config}
\centering
\small
\begin{tabular}{lc}
\toprule
Setting & Value \\
\midrule
Learning rate & $5\!\times\!10^{-6}$ \\
Selected-point budget $K$ & 128 \\
Teacher support size $k$ & 8 \\
Log-probability floor / loss clipping & $-10$ / $+10$ \\
Context length & 65{,}536 \\
Chunk duration / visual sampling rate & 1\,s / 2\,fps \\
Spatial resolution / audio preprocessing & $448^2$ / 16\,kHz mono \\
Decoding parameters & greedy \\
\bottomrule
\end{tabular}
\end{table}
The current artifacts do not provide a complete run manifest for the 7B model or
supervised baselines. We therefore omit unverified checkpoint identifiers,
seeds, hardware time, and baseline-specific coefficients rather than infer them
from script defaults.

\paragraph{Compact teacher targets.}
Teacher scoring evaluates response probabilities under the hint-bearing context.
Let $\mathcal{K}_t^{(\ell)}$ contain the $k$ tokens with greatest probability
under $P_t^{(\ell)}$. The teacher and student distributions are restricted to
these same tokens and renormalized:
\begin{equation}
  \widetilde{P}_t^{(\ell)}(v)
    =\frac{P_t^{(\ell)}(v)}{\sum_{u\in\mathcal{K}_t^{(\ell)}}P_t^{(\ell)}(u)},
  \qquad
  \widetilde{Q}_{\theta,t}^{(\ell)}(v)
    =\frac{Q_{\theta,t}^{(\ell)}(v)}{\sum_{u\in\mathcal{K}_t^{(\ell)}}Q_{\theta,t}^{(\ell)}(u)},
  \label{eq:support}
\end{equation}
for $v\in\mathcal{K}_t^{(\ell)}$. In the implementation, each divergence in
Equation~\ref{eq:loss} is evaluated on these normalized distributions. For a
response of length $L_t$, the retained targets require $O(kL_t)$ probability
values and token identities rather than $O(|\mathcal{V}|L_t)$ values. This
reduces target storage and transfer; it does not eliminate the model computation
needed to obtain those probabilities. Teacher targets and support identities are
fixed during the student update. The model used for subsequent generation and
scoring is synchronized with the updated student before the next iteration.

Loss aggregation includes only valid response positions with scored targets.
Padding, prompt positions, and absent teacher targets do not contribute to the
numerator or denominator. All divergence variants use the same support,
normalization, and token mask. Any numerical floor or clipping applied during
training is reported in Table~\ref{tab:config}, since these operations can alter
the objective at extreme probabilities.

\section{Evaluation details}
\label{app:metrics}

\paragraph{Task coverage and benchmark aggregates.}
OmniMMI includes state grounding (SG), action prediction (AP), multi-turn
dependency reasoning (MD), speaker identification (SI), proactive alerting (PA),
and proactive turn-taking \citep{omnimmi}. Table~\ref{tab:pa} reports only PA,
which measures whether a model alerts the user when an unprompted response is
warranted. Systems without a published numeric PA result are omitted; no
aggregate is reconstructed from the remaining OmniMMI tasks.

StreamingBench reports real-time visual understanding, omni-source understanding,
and contextual understanding \citep{lin2024streaming}. Its overall score follows
the benchmark's task aggregation, rather than an unweighted average of category
means when the categories contain different numbers of tasks.

OVO-Bench reports backward tracing, real-time visual perception, and forward
active responding \citep{li2025ovobench}. Category scores and the aggregate use
the same benchmark release, with the online observation boundary fixed.

\paragraph{Model response decisions.}
For a model using the explicit silence protocol, an empty output or the standalone
silence marker denotes silence. Text accompanying that marker is treated as a
response. Models with a native response gate are scored according to their
released gate output. Native protocol adapters preserve the streaming observation
boundary and map the model's decision into the benchmark's response record.

\paragraph{LiveProBench aggregation.}
LiveProBench presents each video incrementally under a standing instruction and
queries the model at annotated times. Table~\ref{tab:liveprobench} reports all six
tasks; their semantics and representative cases are presented in
Section~\ref{app:case-study}. For ES, let $u$ be the fraction of annotated
alerts covered by a response and $v$ the fraction of annotated silent segments
with no response. Its score is $\sqrt{uv}$. A silent segment is a contiguous
annotated interval, so one interruption invalidates that segment. OR, WR, AA,
and SS use the fraction of episodes whose first response falls inside the
permitted interval; early and missing responses are reported separately. DC uses
final-total exact-count accuracy and does not score whether intermediate replies
align with individual completion times. The headline average is the arithmetic
mean of all six task scores, including DC, rather than a pooled decision accuracy.

\paragraph{Coverage and content.}
Runs covering different examples are not directly compared, and incomplete runs
are not treated as full benchmark results. The tables report aggregate scores;
example counts and inference-failure counts remain in the evaluation artifacts. The optional
semantic-content check in LiveProBench is distinct from its timing score.
OmniMMI, StreamingBench, and OVO-Bench retain their benchmark-specific answer
scoring. Evaluator identities and benchmark versions are recorded alongside the
model checkpoint and input configuration.

\paragraph{StreamingBench system-level comparison.}
Table~\ref{tab:sb-system} reports all 18 subtasks for released systems and for
Live-OPSD with audio enabled. The 30B variant has the highest overall score and
leads the RTVU and OSU group averages among the displayed systems; AURA is
strongest on CU.

\begin{table}[h]
\caption{StreamingBench system-level performance (\%) on all 18 subtasks.
Subtask abbreviations follow the benchmark definition \citep{lin2024streaming};
external comparison results follow the AURA evaluation \citep{lu2026aura}.
Our Live-OPSD variants are evaluated locally with audio.
For external systems, group and overall scores follow the source aggregation
and need not equal unweighted means of the displayed subtasks. ``All'' is not
the mean of the three group averages. Boldface marks the best group or overall
score.}
\label{tab:sb-system}
\centering
\footnotesize
\setlength{\tabcolsep}{3pt}
\resizebox{\textwidth}{!}{%
\begin{tabular}{l*{22}{c}}
\toprule
& \multicolumn{11}{c}{RTVU}
& \multicolumn{5}{c}{OSU}
& \multicolumn{5}{c}{CU} & \\
\cmidrule(lr){2-12}\cmidrule(lr){13-17}\cmidrule(lr){18-22}
Model & OP & CR & CS & ATP & EU & TR & PR & SU & ACP & CT & Avg
& ER & SCU & SD & MA & Avg & MCU & ACU & SQA & PO & Avg & All \\
\midrule
Gemini-1.5-Pro & 79.0 & 80.5 & 83.5 & 79.7 & 80.0 & 84.7 & 77.8 & 64.2 & 72.0 & 48.7 & 75.7 & 46.8 & 39.6 & 74.9 & 80.0 & 60.2 & 51.4 & 40.7 & 54.8 & 45.1 & 48.7 & 67.10 \\
GPT-4o & 77.1 & 80.5 & 83.9 & 76.5 & 70.2 & 83.8 & 66.7 & 62.2 & 69.1 & 49.2 & 73.3 & 41.2 & 37.2 & 43.6 & 56.0 & 44.5 & 41.2 & 38.4 & 32.8 & 56.9 & 38.7 & 60.20 \\
StreamAgent & 79.6 & 78.3 & 79.3 & 75.9 & 74.7 & 76.9 & 82.9 & 66.3 & 73.7 & 55.4 & 74.3 & 35.9 & 26.3 & 38.9 & 44.0 & 36.3 & 39.7 & 30.3 & 39.6 & 28.9 & 34.6 & 57.00 \\
ViSpeak & 79.8 & 71.1 & 81.4 & 78.8 & 74.5 & 70.1 & 63.9 & 64.2 & 71.4 & 28.0 & 70.4 & 47.2 & 56.4 & 61.6 & 81.2 & 61.6 & 49.2 & 36.4 & 39.2 & 50.8 & 43.9 & 62.60 \\
MiniCPM-o-4.5 & 80.9 & 88.3 & 82.0 & 85.0 & 78.9 & 85.4 & 78.7 & 67.1 & 75.6 & 56.0 & 78.2 & 44.4 & 28.0 & 41.6 & 54.4 & 42.1 & 35.2 & 44.0 & 51.2 & 47.6 & 44.5 & 62.70 \\
AURA & 87.5 & 84.4 & 93.4 & 89.5 & 83.2 & 87.9 & 86.1 & 76.4 & 82.4 & 47.7 & 83.2 & 56.4 & 48.4 & 62.8 & 80.4 & 62.0 & 54.8 & 70.8 & 57.2 & 53.2 & \textbf{59.0} & 73.10 \\
\midrule
Live-OPSD-7B & 86.9 & 84.4 & 90.2 & 90.8 & 85.1 & 90.3 & 85.2 & 77.2 & 80.2 & 43.0 & 81.3 & 63.2 & 51.2 & 74.8 & 89.6 & 69.7 & 41.2 & 57.2 & 48.0 & 35.2 & 45.4 & 70.80 \\
Live-OPSD-30B & 90.2 & 85.9 & 94.3 & 90.5 & 86.3 & 90.0 & 88.9 & 80.1 & 83.9 & 42.5 & \textbf{83.3} & 64.0 & 64.8 & 82.0 & 92.0 & \textbf{75.7} & 48.4 & 68.4 & 58.4 & 56.8 & 58.0 & \textbf{76.00} \\
\bottomrule
\end{tabular}%
}
\end{table}

\paragraph{OVO-Bench system-level comparison.}
Table~\ref{tab:ovo-system} reports the full OVO-Bench comparison, with per-subtask
scores for the released systems and for our backbones and their Live-OPSD
variants. Live-OPSD leads real-time visual perception and backward tracing, while
forward active responding stays below the strongest released system.

\begin{table}[h]
\caption{OVO-Bench system-level performance (\%) on all 12 subtasks.
Subtask abbreviations follow the benchmark definition \citep{li2025ovobench};
external comparison results follow the AURA evaluation \citep{lu2026aura}.
Our Qwen-Omni backbones and their Live-OPSD variants are evaluated locally,
video only. ``All'' is the macro-average of the RTVP, BT,
and FAR group averages. Boldface marks the best group or overall score.}
\label{tab:ovo-system}
\centering
\footnotesize
\setlength{\tabcolsep}{3pt}
\resizebox{\textwidth}{!}{%
\begin{tabular}{l*{16}{c}}
\toprule
& \multicolumn{7}{c}{RTVP}
& \multicolumn{4}{c}{BT}
& \multicolumn{4}{c}{FAR} & \\
\cmidrule(lr){2-8}\cmidrule(lr){9-12}\cmidrule(lr){13-16}
Model & OCR & ACR & ATR & STU & FPD & OJR & Avg
& EPM & ASI & HLD & Avg & REC & SSR & CRR & Avg & All \\
\midrule
Gemini-1.5-Pro & 85.9 & 67.0 & 79.3 & 58.4 & 63.4 & 62.0 & 69.3 & 58.6 & 76.4 & 52.6 & 62.5 & 35.5 & 74.2 & 61.7 & \textbf{57.2} & 63.00 \\
GPT-4o & 69.8 & 64.2 & 71.6 & 51.1 & 70.3 & 59.8 & 64.5 & 57.9 & 75.7 & 48.7 & 60.8 & 73.2 & 27.6 & 59.4 & 53.4 & 59.50 \\
StreamAgent & 71.2 & 53.2 & 63.6 & 53.9 & 67.3 & 58.7 & 61.3 & 54.8 & 58.1 & 25.8 & 41.7 & 35.9 & 48.4 & 52.0 & 45.4 & 49.40 \\
Streamo-7B & 77.2 & 66.1 & 76.7 & 45.5 & 66.3 & 72.8 & 67.4 & 55.6 & 58.1 & 33.9 & 49.2 & 30.8 & 57.6 & 82.5 & 57.0 & 57.90 \\
ViSpeak & 75.2 & 58.7 & 71.6 & 51.1 & 74.3 & 66.9 & 66.3 & 59.9 & 48.7 & 64.0 & 57.5 & 33.8 & 68.5 & 60.4 & 54.3 & 61.10 \\
MiniCPM-o-4.5 & 73.2 & 65.1 & 75.0 & 59.0 & 69.3 & 62.5 & 67.3 & 55.2 & 66.9 & 45.7 & 55.9 & 43.0 & 66.0 & 58.8 & 55.9 & 59.70 \\
AURA & 89.9 & 79.8 & 80.2 & 70.2 & 77.2 & 81.5 & 79.8 & 54.9 & 67.6 & 58.6 & 60.4 & 30.4 & 75.8 & 61.3 & 55.8 & 65.30 \\
\midrule
Qwen2.5-Omni-7B & 73.2 & 50.5 & 71.6 & 50.0 & 67.3 & 61.4 & 62.3 & 55.6 & 56.8 & 9.1 & 40.5 & 34.0 & 57.2 & 53.3 & 48.2 & 50.30 \\
Live-OPSD-7B & 90.6 & 79.8 & 78.5 & 61.2 & 79.2 & 77.7 & 77.8 & 62.0 & 64.9 & 41.4 & 56.1 & 30.4 & 68.5 & 57.1 & 52.0 & 61.97 \\
Qwen3-Omni-30B-A3B & 79.2 & 60.6 & 74.1 & 57.9 & 76.2 & 65.8 & 69.0 & 59.3 & 65.5 & 48.9 & 57.9 & 38.8 & 69.3 & 60.0 & 56.1 & 61.00 \\
Live-OPSD-30B & 92.6 & 88.1 & 84.5 & 64.0 & 82.2 & 79.4 & \textbf{81.8} & 66.7 & 67.6 & 61.3 & \textbf{65.2} & 34.8 & 67.3 & 49.6 & 50.6 & \textbf{65.84} \\
\bottomrule
\end{tabular}%
}
\end{table}

\paragraph{Held-out ablation dynamics.}
During the divergence and DAS ablations, proactive decision quality is tracked on
a fixed held-out set of 1458 decision points, identical across all legs. Because
this set does not depend on each run's on-policy batch composition, the learning
curves are directly comparable across legs. Figure~\ref{fig:ablation-curves}
plots proactive $F_1$, event-miss rate, false-alarm rate, and per-episode exact
match over the first 110 training steps.
These trajectories complement the final-checkpoint scores in
Table~\ref{tab:ablation}.

\section{Case Study}
\label{app:case-study}
We use LiveProBench as a concrete case study of streaming response decisions. The
benchmark covers six complementary behaviors:
\begin{itemize}
\setlength{\itemsep}{0pt}
\setlength{\parskip}{0pt}
\item \textbf{Event Sensitivity (ES)} requires a response when the first relevant
visual evidence appears and silence throughout annotated negative segments.
\item \textbf{Onset Response (OR)} anchors the target to the annotated event onset
and allows a 3-second grace period after that onset.
\item \textbf{Window Response (WR)} accepts the first response anywhere in a
broader annotated interval, allowing temporal ambiguity around the event.
\item \textbf{Anomaly Alert (AA)} tests semantically ambiguous events whose alert
condition is less crisply specified than a single explicit action boundary.
\item \textbf{Sustained Silence (SS)} places the target after an especially long
annotated silence window and requires no response throughout that prefix.
\item \textbf{Duplicate Counting (DC)} presents repeated occurrences of one action
and scores only the final total count; exact response timing at individual
completions is not evaluated.
\end{itemize}
Figure~\ref{fig:liveprobench-cases} visualizes three representative patterns using
annotation-only examples: ES, a first-response window represented by WR, and
final-total counting in DC. OR, AA, and SS share the same first-response-window
layout as WR and are not shown separately; they differ mainly in query type and in
window length, which departs substantially from WR's broad interval. OR anchors a
short grace window to an explicit onset, AA queries a semantically ambiguous
anomaly, and SS follows an especially long silence prefix.

\begin{figure*}[h]
\centering
\includegraphics[width=\textwidth]{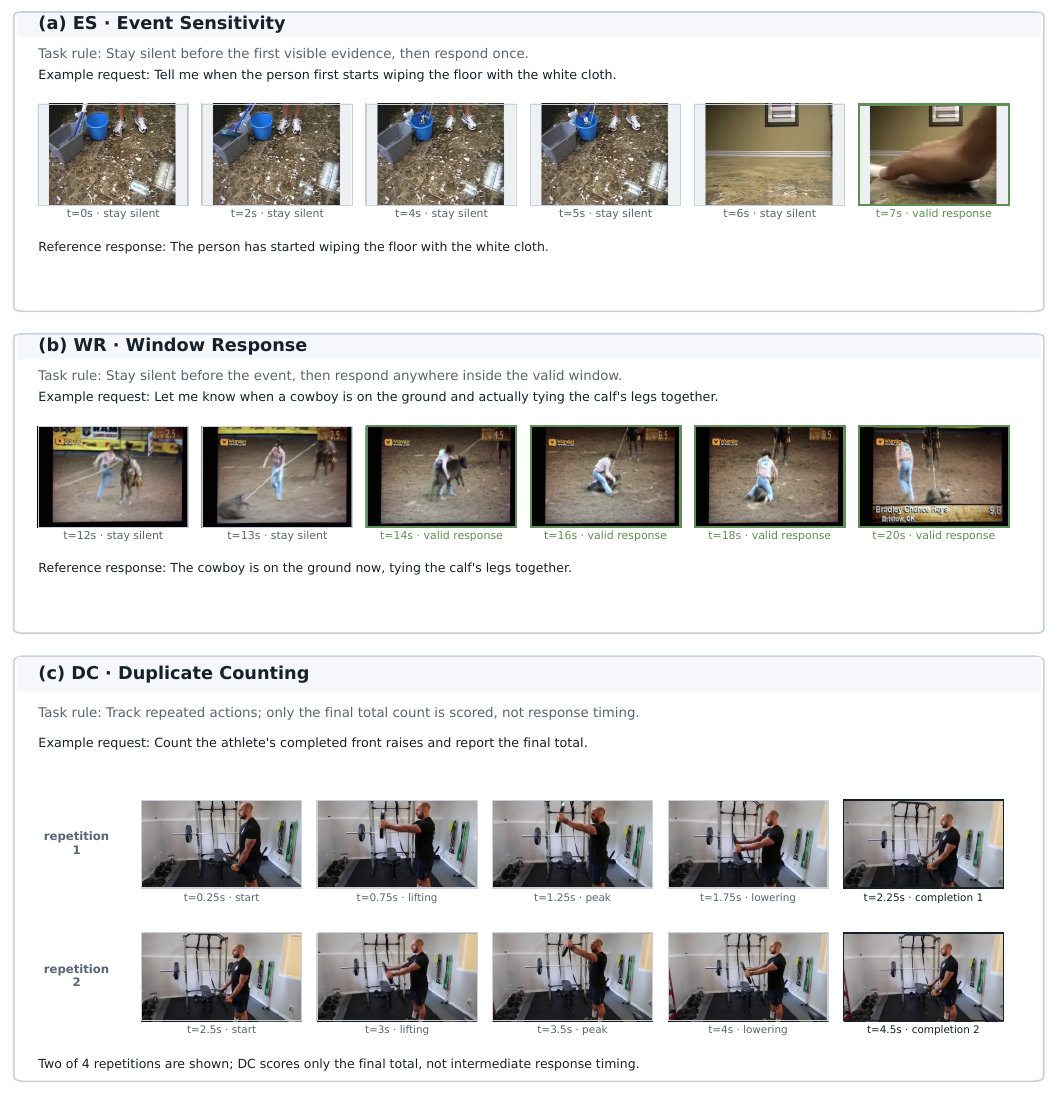}
\caption{Representative LiveProBench task formats from the released annotations.
ES illustrates response--silence balance; WR represents the first-response-window
format shared by OR, WR, AA, and SS; and DC illustrates final-total counting across
repeated events. ES and WR each show six complete video frames; DC uses ten
frames spanning two complete action cycles. Frame labels identify silence,
valid response moments, and action phases; green borders mark valid ES/WR responses.
No model predictions are shown.}
\label{fig:liveprobench-cases}
\end{figure*}

Figure~\ref{fig:anomaly-cases} compares the training methods on individual cases
from the same released annotations. Each case shows the sampled frames, the
annotated valid response window, and the first response emitted by Live-OPSD and
the streaming-SFT baselines, using the audio-disabled predictions behind
Table~\ref{tab:liveprobench}. Live-OPSD responds inside the window in every case,
whereas the SFT baselines either stay silent or respond prematurely.

\begin{figure*}[h]
\centering
\includegraphics[width=\textwidth]{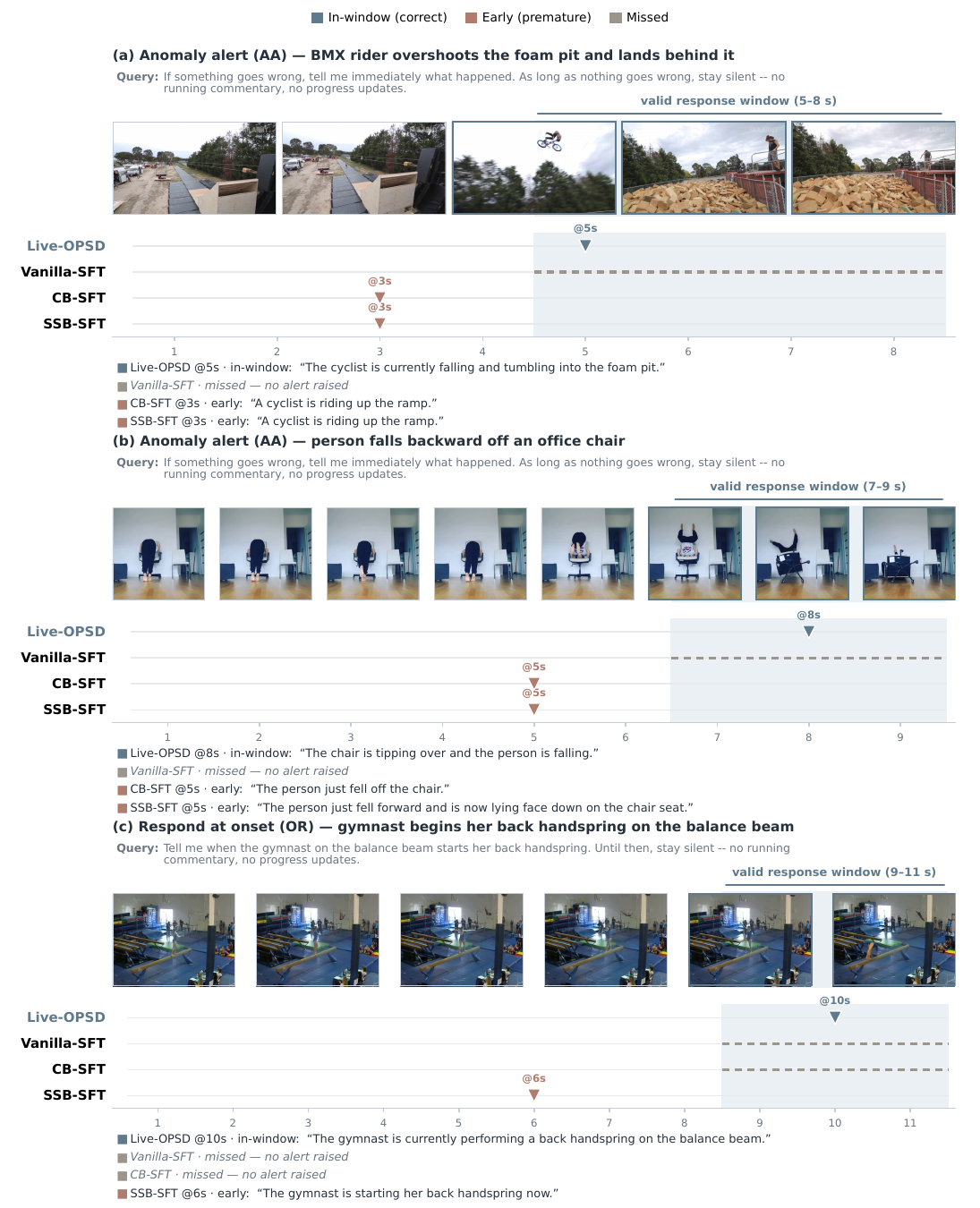}
\caption{Case comparison of Live-OPSD, Vanilla-SFT, CB-SFT, and SSB-SFT on
released LiveProBench streams. Each panel shows the sampled video frames, the
annotated valid response window, and each method's first response, colored by
verdict: in-window (correct), early (premature), or missed.}
\label{fig:anomaly-cases}
\end{figure*}

\clearpage
\section{Decision-level analysis of Live-OPSD}
\label{app:theory}

Live-OPSD uses privileged context to construct a teaching signal, on-policy
distillation to transfer that signal, and DAS to choose where to apply it.
The analysis asks when these operations can support better response decisions.
Let $(S,H,B)\sim\rho$ denote a recorded streaming context, its training-time
hint, and its supervisory decision, with $B\in\{0,1\}$. The student observes
$S$, whereas the teacher also observes $H$. All decision risks below measure
disagreement with $B$. We first characterize the information supplied by a hint,
then relate the actual teacher--student discrepancy to decision risk, and finally
identify when disagreement-based selection enriches contexts with a teaching
advantage.

\subsection{Privileged context and decision ambiguity}
\label{app:information}

A shared parameter set need not imply identical decision information. The first
result characterizes the ambiguity that can be resolved by the teacher's
additional context. Write
$u(s)=\Pr(B=1\mid S=s)$ and
$v(s,h)=\Pr(B=1\mid S=s,H=h)$ for the two conditional probabilities of a
supervised response.

\begin{proposition}[Information advantage of privileged context]
\label{prop:privileged-information}
Let
\[
  R_{\mathrm{vis}}^*=\inf_g\Pr(g(S)\ne B),
  \qquad
  R_{\mathrm{priv}}^*=\inf_f\Pr(f(S,H)\ne B),
\]
where the infima range over all measurable binary decision rules. Then
\[
  R_{\mathrm{vis}}^*-R_{\mathrm{priv}}^*
  =\mathbb{E}\!\left[
       \left|v(S,H)-\tfrac12\right|
       -\left|u(S)-\tfrac12\right|\right]\ge0.
\]
The inequality is strict if, for a set of contexts $s$ of positive probability,
$v(s,H)$ lies on both sides of $1/2$ with positive conditional probability.
\end{proposition}

\begin{proof}
For a Bernoulli supervisory decision with response probability $u$, choosing
silence incurs disagreement probability $u$, and choosing a response incurs
$1-u$. The minimum is $\min\{u,1-u\}=1/2-|u-1/2|$. Applying this calculation
under the two information sets gives
\[
  R_{\mathrm{vis}}^*=\tfrac12-\mathbb{E}|u(S)-\tfrac12|,
  \qquad
  R_{\mathrm{priv}}^*=\tfrac12-\mathbb{E}|v(S,H)-\tfrac12|.
\]
The tower property gives $u(S)=\mathbb{E}[v(S,H)\mid S]$. Conditional Jensen's
inequality therefore implies
\[
  |u(S)-\tfrac12|
    \le\mathbb{E}[|v(S,H)-\tfrac12|\mid S].
\]
Subtracting the two minimum risks proves the identity and its nonnegativity.
The absolute-value inequality is strict when its argument has both positive and
negative values with positive conditional probability, yielding the stated
strictness condition.
\end{proof}

The strictness condition identifies a source of teaching information: the hint
distinguishes supervisory decisions that the visible context alone leaves
ambiguous. When $H$ is a deterministic function of $S$, the two optimal risks
are equal. Such a hint can still make an evidence--decision relationship
explicit for a finite model, but that benefit concerns how the model uses
information rather than an increase in the information available. Thus, the
hint's information advantage and the current teacher's decision advantage are
different quantities. The transfer analysis uses the latter.

\subsection{From on-policy distillation to response decisions}
\label{app:decision-transfer}

Fix a teacher and student at a training iterate. Let $P(\cdot\mid s,h)$ and
$Q(\cdot\mid s)$ be their autoregressive distributions over continuations of at
most $L$ tokens. Ended sequences are padded with an absorbing end symbol to give
a common horizon; both conditionals put all mass on that symbol once a turn has
ended, so padded positions contribute no discrepancy. The binary response
decision $d(Y)$ is determined by the complete continuation, as in
Section~\ref{sec:formulation}. Define
\[
  R(Q)=\mathbb{E}_\rho\Pr_{Y\sim Q(\cdot\mid S)}[d(Y)\ne B],
  \qquad
  R(P)=\mathbb{E}_\rho\Pr_{Y\sim P(\cdot\mid S,H)}[d(Y)\ne B].
\]
Although training only scores the teacher, its conditional probabilities define
such a continuation distribution. This lets us ask whether matching those
probabilities along student-generated prefixes also controls response decisions.

At a fixed context, let $p_\ell(\cdot\mid y_{<\ell})$ and
$q_\ell(\cdot\mid y_{<\ell})$ denote the teacher and student next-token
distributions. Define the cumulative on-policy reverse KL
\[
  \mathcal{K}_{\mathrm r}(P,Q)=\mathbb{E}_\rho\mathbb{E}_{Y\sim Q}
    \sum_{\ell=1}^{L}
      \KL\!\left(q_\ell(\cdot\mid Y_{<\ell})
             \Vert p_\ell(\cdot\mid Y_{<\ell})\right).
\]
The context dependence is implicit in $p_\ell$ and $q_\ell$. Logarithms are
natural, and the continuation expectation follows the student, as in the
on-policy objective. Write
$\operatorname{TV}(U,V)=\tfrac12\|U-V\|_1$ for total variation.

\begin{proposition}[Decision transfer along student-generated prefixes]
\label{prop:decision-transfer}
For the finite-horizon continuation distributions above,
\[
  R(Q)\le R(P)+\sqrt{\frac{L\,\mathcal{K}_{\mathrm r}(P,Q)}{2}}.
\]
\end{proposition}

\begin{proof}
Fix $(s,h,b)$ and introduce hybrid continuation distributions
\[
  M_j(y_{1:L})=
    \prod_{\ell=1}^{j}q_\ell(y_\ell\mid y_{<\ell})
    \prod_{\ell=j+1}^{L}p_\ell(y_\ell\mid y_{<\ell}),
  \qquad j=0,\ldots,L.
\]
Thus $M_0=P$ and $M_L=Q$. The distributions $M_{j-1}$ and $M_j$ share a
student-generated prefix of length $j-1$ and the same teacher continuation after
token $j$. Summing over this common suffix gives
\[
  \operatorname{TV}(M_{j-1},M_j)
   =\mathbb{E}_{Y_{<j}\sim Q}
       \operatorname{TV}\!\left(p_j(\cdot\mid Y_{<j}),
                                q_j(\cdot\mid Y_{<j})\right).
\]
The triangle inequality and Pinsker's inequality yield
\[
  \operatorname{TV}(P,Q)
    \le\mathbb{E}_{Y\sim Q}\sum_{j=1}^{L}
       \sqrt{\frac12\KL\!\left(q_j(\cdot\mid Y_{<j})
                         \Vert p_j(\cdot\mid Y_{<j})\right)}.
\]
For the event $\{Y:d(Y)\ne b\}$, the difference between its student and teacher
probabilities is at most $\operatorname{TV}(P,Q)$. Averaging over contexts and
applying Cauchy--Schwarz to the $L$ prefix contributions proves the result.
\end{proof}

The prefix distribution in this result is the student's own. The discrepancy is
therefore evaluated where the student makes the token choices that determine its
action. Matching only reference continuations controls a different expectation
and requires an additional coverage condition. The result concerns continuations
at recorded contexts under $\rho$, not full model-generated interaction histories.

The bound charges one discrepancy term per response position. Response decisions
are usually settled before a turn ends. Let
\[
  \nu(Y)=\min\bigl\{j\le L:\ d\text{ is constant on all continuations agreeing}
    \text{ with }Y\text{ in the first }j\text{ tokens}\bigr\}
\]
be the position at which the decision is resolved, a stopping time for the token
filtration, and write
\[
  \bar\nu=\mathbb{E}_\rho\mathbb{E}_{Y\sim Q}\,\nu(Y),
  \qquad
  \mathcal{K}_{\mathrm r,\nu}(P,Q)=
    \mathbb{E}_\rho\mathbb{E}_{Y\sim Q}
    \sum_{j=1}^{\nu(Y)}
      \KL\!\left(q_j(\cdot\mid Y_{<j})
             \Vert p_j(\cdot\mid Y_{<j})\right).
\]

\begin{proposition}[Decision transfer over the resolving prefix]
\label{prop:resolution-horizon}
Under the assumptions of Proposition~\ref{prop:decision-transfer},
\[
  R(Q)\le R(P)+
    \sqrt{\frac{\bar\nu\,\mathcal{K}_{\mathrm r,\nu}(P,Q)}{2}}.
\]
Since $\nu\le L$, this implies Proposition~\ref{prop:decision-transfer}.
\end{proposition}

\begin{proof}
Fix $(s,h,b)$ and let $A=\{y:d(y)\ne b\}$. By construction $d$ is a function of
$y_{1:\nu(y)}$, and the event $\{\nu\ge j\}$ is determined by $y_{<j}$. With the
hybrid distributions of the previous proof,
\[
  M_{j-1}(A)-M_j(A)
   =\mathbb{E}_{Y_{<j}\sim Q}
      \sum_{y_j}\bigl[p_j(y_j\mid Y_{<j})-q_j(y_j\mid Y_{<j})\bigr]
      \Phi(Y_{<j},y_j),
\]
where $\Phi(y_{\le j})$ is the probability of $A$ under the teacher's
continuation from $y_{\le j}$. If $\nu\le j-1$, the prefix already determines
$d$, so the inner sum vanishes. If $\nu\ge j$, its magnitude is bounded by
$\operatorname{TV}(p_j,q_j)$. Telescoping, applying Pinsker's inequality, and
using Cauchy--Schwarz with the weights $\mathbb{1}[\nu\ge j]$ proves the result.
\end{proof}

Under the explicit silence protocol, a leading token other than the silence
marker resolves the decision at the first position. The silent branch is
resolved once the turn ends without accompanying text, following the scoring rule
of Appendix~\ref{app:metrics}. The resolving prefix is therefore short compared
with a full response, and $\bar\nu$ is measurable from the current policy.
Response tokens beyond $\nu$ still receive supervision in
Equation~\ref{eq:loss}; they affect response content rather than the binary
response decision. Both corollaries below hold with
$(L,\mathcal{K}_{\mathrm r})$ replaced by
$(\bar\nu,\mathcal{K}_{\mathrm r,\nu})$.

\begin{corollary}[Transfer of a teacher's decision advantage]
\label{cor:decision-improvement}
Let $Q_0$ be a reference student and suppose the fixed teacher has decision
advantage $\Delta=R(Q_0)-R(P)>0$ under the same context distribution. For a
candidate student $Q$,
\[
  R(Q)-R(Q_0)\le-\Delta+
    \sqrt{\frac{L\,\mathcal{K}_{\mathrm r}(P,Q)}{2}}.
\]
In particular, $\mathcal{K}_{\mathrm r}(P,Q)<2\Delta^2/L$ is sufficient for
$R(Q)<R(Q_0)$.
\end{corollary}

\begin{proof}
Subtract $R(Q_0)=R(P)+\Delta$ from
Proposition~\ref{prop:decision-transfer}. The right-hand side is negative under
the stated condition.
\end{proof}

The sufficient condition compares the teacher's decision advantage with the
residual transfer discrepancy. In Live-OPSD, $Q_0$ generates the current batch
and $P$ is its hinted counterpart. Equation~\ref{eq:loss} evaluates candidate
updates on fixed selected samples, whereas $\mathcal{K}_{\mathrm r}(P,Q)$ averages
full-vocabulary discrepancies over fresh continuations from $Q$. Connecting the
two requires the selected samples and retained support to cover the relevant
contexts, prefixes, and vocabulary. The corollary is a distribution-level
sufficient condition, not a guarantee for every optimization step.

\begin{corollary}[Residual discrepancy from privileged information]
\label{cor:privileged-residual}
For any student whose continuation depends only on $S$,
\[
  \mathcal{K}_{\mathrm r}(P,Q)\ge
    \frac{2\bigl[R_{\mathrm{vis}}^*-R(P)\bigr]_+^2}{L},
  \qquad [x]_+=\max\{x,0\}.
\]
\end{corollary}

\begin{proof}
Randomizing the response decision cannot reduce conditional disagreement below
the best deterministic decision, so $R(Q)\ge R_{\mathrm{vis}}^*$.
Combine this with Proposition~\ref{prop:decision-transfer} and rearrange when
$R_{\mathrm{vis}}^*>R(P)$. Otherwise the claim follows from nonnegativity of KL.
\end{proof}

When the hint reveals decision-relevant information unavailable in the visible
context, exact teacher--student agreement can be unattainable. The objective is
to improve decisions from the student's information set; zero distillation loss
is not required for useful transfer.

The discrepancies above use the full vocabulary, whereas the implementation
replaces each divergence in Equation~\ref{eq:loss} with its retained-support
counterpart from Equation~\ref{eq:support}. The remaining result bounds this
difference by the discarded probability masses.

\begin{proposition}[From the truncated objective to full-vocabulary transfer]
\label{prop:truncation-gap}
At each position let $m_j^P$ and $m_j^Q$ be the teacher and student masses
retained by the support $\mathcal{K}_j$ of Equation~\ref{eq:support}, and let
$\widetilde p_j,\widetilde q_j$ be the renormalized distributions. Then
\[
  \operatorname{TV}(p_j,q_j)
    \le\sqrt{\frac12\KL(\widetilde q_j\Vert\widetilde p_j)}
       +(1-m_j^P)+(1-m_j^Q),
\]
and consequently
\[
  R(Q)\le R(P)
    +\sqrt{\frac{\bar\nu\,\widetilde{\mathcal{K}}_{\mathrm r,\nu}}{2}}
    +\mathbb{E}_\rho\mathbb{E}_{Y\sim Q}\sum_{j=1}^{\nu(Y)}
       \bigl[(1-m_j^P)+(1-m_j^Q)\bigr],
\]
where $\widetilde{\mathcal{K}}_{\mathrm r,\nu}$ uses reverse KL between the
renormalized distributions on the retained support.
\end{proposition}

\begin{proof}
Split $\|p_j-q_j\|_1$ at $\mathcal{K}_j$. The two tails outside the support
contribute at most $(1-m_j^P)+(1-m_j^Q)$. Inside it,
$p_j=m_j^P\widetilde p_j$ and $q_j=m_j^Q\widetilde q_j$. Summing the pointwise
triangle inequality and using
$|m_j^P-m_j^Q|\le(1-m_j^P)+(1-m_j^Q)$ gives
\[
  \operatorname{TV}(p_j,q_j)
  \le\operatorname{TV}(\widetilde p_j,\widetilde q_j)
       +(1-m_j^P)+(1-m_j^Q).
\]
Pinsker's inequality proves the first claim. Substituting it into the telescoped
bound for Proposition~\ref{prop:resolution-horizon} and applying
Cauchy--Schwarz to the reverse-KL terms proves the second.
\end{proof}

With $\mathcal D=\mathcal D_{\mathrm{rKL}}$, the implemented form of
Equation~\ref{eq:loss} minimizes the retained-support reverse-KL terms in this
bound. The discarded masses isolate the effect of truncation and are computable
from the teacher and student forward passes. The remaining difference between
the bound and the implemented update is that training uses fixed samples at
selected contexts with global token-mean normalization rather than an idealized
expectation over all contexts and continuations.

\subsection{Which contexts decision-aware selection emphasizes}
\label{app:das-analysis}

A sampled disagreement contains information about the context that produced it.
Whether that context is useful for distillation also depends on the teacher.
To separate these questions, fix the current policy $Q_\theta$ and teacher $P$,
and let $Z=(S,H,B)\sim\rho$. Define
\[
  e_\theta(z)=\Pr_{Y\sim Q_\theta(\cdot\mid s)}[d(Y)\ne b],
  \qquad
  \bar e_\theta=\mathbb{E}_{Z\sim\rho}e_\theta(Z).
\]
Here $e_\theta(z)$ is the probability of disagreement on a fresh continuation at
that context, while the indicator $\delta_t$ used by DAS is one sampled outcome.
Write $e_P(z)=\Pr_{Y\sim P(\cdot\mid s,h)}[d(Y)\ne b]$ for the teacher's
conditional disagreement and $g_\theta(z)=e_\theta(z)-e_P(z)$ for its
context-specific decision advantage. A positive $g_\theta(z)$ means that the
teacher agrees with supervision more often than the current student there. Its
mean $\mathbb{E}_\rho g_\theta(Z)$ is the decision advantage $\Delta$ of
Corollary~\ref{cor:decision-improvement} at the current policy.

\begin{proposition}[Context enrichment and teaching advantage]
\label{prop:das-enrichment}
Suppose $\bar e_\theta>0$. Draw a context from $\rho$, generate a continuation
from the current policy, and retain the context when its decision disagrees
with supervision. The conditional context distribution is
\[
  \rho_{\mathrm{inc}}(dz)
    =\frac{e_\theta(z)}{\bar e_\theta}\,\rho(dz).
\]
Its expected disagreement probability on a fresh independent continuation is
\[
  \mathbb{E}_{\rho_{\mathrm{inc}}}e_\theta(Z)
    =\bar e_\theta+
      \frac{\operatorname{Var}_\rho(e_\theta(Z))}{\bar e_\theta}
    \ge\bar e_\theta,
\]
with strict inequality whenever $e_\theta(Z)$ is nonconstant under $\rho$.
The expected teacher advantage satisfies
\[
  \mathbb{E}_{\rho_{\mathrm{inc}}}g_\theta(Z)
    =\mathbb{E}_\rho g_\theta(Z)+
      \frac{\operatorname{Cov}_\rho(e_\theta(Z),g_\theta(Z))}{\bar e_\theta}.
\]
Thus, the inconsistent pool has a greater mean teacher advantage precisely when
the covariance is positive.
\end{proposition}

\begin{proof}
Let $D$ denote disagreement on the continuation used for selection. For any
measurable set of contexts $A$,
\[
  \Pr(Z\in A,D)=\int_A e_\theta(z)\,\rho(dz),
  \qquad \Pr(D)=\bar e_\theta.
\]
Bayes' rule gives the stated conditional distribution. For any bounded
context-level function $f$,
\begin{align*}
  \mathbb{E}_{\rho_{\mathrm{inc}}}f(Z)
    &=\frac{\mathbb{E}_\rho[e_\theta(Z)f(Z)]}{\bar e_\theta}\\
    &=\mathbb{E}_\rho f(Z)+
      \frac{\operatorname{Cov}_\rho(e_\theta(Z),f(Z))}{\bar e_\theta}.
\end{align*}
An independent continuation at a retained context still has disagreement
probability $e_\theta(z)$. Taking $f=e_\theta$ gives the variance identity, and
the variance is positive unless $e_\theta(Z)$ is almost surely constant.
Taking $f=g_\theta$ proves the identity for teacher advantage. Since
$\bar e_\theta>0$, the change in that advantage has the sign of the covariance.
\end{proof}

\begin{corollary}[Quota relative to the current disagreement rate]
\label{cor:quota}
Suppose $0<\bar e_\theta<1$ and let
$\rho_{\mathrm{con}}(dz)=\frac{1-e_\theta(z)}{1-\bar e_\theta}\,\rho(dz)$ denote
the complementary consistent pool. For
$\rho_\lambda=\lambda\rho_{\mathrm{inc}}+(1-\lambda)\rho_{\mathrm{con}}$ with
$\lambda\in[0,1]$,
\[
  \mathbb{E}_{\rho_\lambda}g_\theta(Z)-\mathbb{E}_\rho g_\theta(Z)
    =\frac{(\lambda-\bar e_\theta)\,
           \operatorname{Cov}_\rho(e_\theta(Z),g_\theta(Z))}
          {\bar e_\theta(1-\bar e_\theta)}.
\]
When the covariance is positive, the mixture has a greater mean teacher advantage
than $\rho$ exactly when $\lambda>\bar e_\theta$, and $\lambda=\bar e_\theta$
reproduces $\rho$.
\end{corollary}

\begin{proof}
Write $C=\operatorname{Cov}_\rho(e_\theta(Z),g_\theta(Z))$. Since
$\mathbb{E}_\rho[(1-e_\theta)g_\theta]
 =(1-\bar e_\theta)\mathbb{E}_\rho g_\theta-C$, dividing by $1-\bar e_\theta$
gives $\mathbb{E}_{\rho_{\mathrm{con}}}g_\theta
 =\mathbb{E}_\rho g_\theta-C/(1-\bar e_\theta)$, while
Proposition~\ref{prop:das-enrichment} gives
$\mathbb{E}_{\rho_{\mathrm{inc}}}g_\theta
 =\mathbb{E}_\rho g_\theta+C/\bar e_\theta$. Taking the $\lambda$-weighted
combination and collecting terms yields the identity. At
$\lambda=\bar e_\theta$ it vanishes, consistent with
$\rho=\bar e_\theta\rho_{\mathrm{inc}}+(1-\bar e_\theta)\rho_{\mathrm{con}}$.
\end{proof}

The covariance condition connects the selection signal to the distillation
signal. DAS favors useful teaching contexts when the hinted model's advantage
is greater where the current student is more likely to disagree with
supervision. If both models struggle on the same contexts, disagreement alone
need not indicate a stronger teaching signal. The covariance describes the
quality of the selected targets. Translating that advantage into a student
improvement also depends on transfer fidelity, as expressed in
Corollary~\ref{cor:decision-improvement}.

The proposition characterizes the inconsistent candidate pool before DAS
balances the two disagreement directions and adds consistent points. The same
identities hold within each supervisory class by conditioning $\rho$ on $B=b$,
provided that class has positive disagreement probability. Direction balancing
and the quota then determine how these pools enter the selected set, and
response lengths determine their weights in Equation~\ref{eq:loss}.
Corollary~\ref{cor:quota} characterizes the inconsistent share of an idealized
mixture relative to the current disagreement rate $\bar e_\theta$. In the
implemented procedure, the quota $q=\max\{1,\lfloor\alpha K\rfloor\}$ caps the
number of inconsistent points before the remaining groups are filled by
priority. The current run artifacts do not retain the pool counts needed to show
that the realized inconsistent share exceeds $\bar e_\theta$ at every update.
Moreover, the corollary samples the remaining points in proportion to
$1-e_\theta$, whereas DAS uses a deterministic group priority, so
$\rho_\lambda$ is an analytical approximation rather than the exact selected-set
distribution. As the policy changes, $e_\theta$ and its induced context weighting
change with it. The fixed-selection control in Section~\ref{sec:ablation}
removes this policy-dependent allocation while preserving each point's visible
history.

\end{document}